\documentclass{article}

\PassOptionsToPackage{numbers, compress}{natbib}

\usepackage[final]{neurips_2023}
\usepackage{algorithmic}
\usepackage{algorithm}
\usepackage{graphicx}
\usepackage{caption}
\usepackage{subcaption}

\usepackage[utf8]{inputenc} 
\usepackage[T1]{fontenc}    
\usepackage{hyperref}       
\usepackage{url}            
\usepackage{booktabs}       
\usepackage{amsfonts}       
\usepackage{nicefrac}       
\usepackage{microtype}      
\usepackage[dvipsnames]{xcolor}

\usepackage{amsmath}
\usepackage{amssymb}
\usepackage{amsthm}
\newtheorem{lemma}{Lemma}
\newtheorem{theorem}{Theorem}
\newtheorem{assumption}{Assumption}

\newtheorem{corollary}{Corollary}
\newtheorem{definition}{Definition}
\newtheorem{claim}{Claim}
\usepackage{comment}
\usepackage{mathtools}
\usepackage{bm}
\newcommand{\cE}{\mathcal{E}}
\newcommand{\cK}{\mathcal{K}}
\newcommand{\cC}{\mathcal{C}}
\newcommand{\bP}{\mathbb{P}}
\newcommand{\Otilde}{\Tilde{O}}
\newcommand{\iidsim}{\stackrel{i.i.d.}{\sim}}
\newcommand{\dotp}[2]{\left\langle #1, #2 \right\rangle}

\newcommand{\pistar}{\pi^{\star}}
\newcommand{\bR}{\mathbb{R}}
\newcommand{\vol}{\mathsf{vol}}

\newcommand{\cB}{\mathcal{B}}
\newcommand{\KSD}[3]{\mathsf{KSD}_{#1}(#2||#3)}
\newcommand{\KSDsq}[3]{\mathsf{KSD}^2_{#1}(#2||#3)}
\newcommand{\cF}{\mathcal{F}}
\newcommand{\bE}{\mathbb{E}}
\newcommand{\vx}{\mathbf{x}}

\newcommand{\Uniform}{\mathsf{Uniform}}

\newcommand{\dd}{\mathsf{d}}

\newcommand{\bN}{\mathbb{N}}

\newcommand{\ouralg}{VP-SVGD}
\newcommand{\ouralgnew}{GB-SVGD}

\newcommand{\cN}{\mathcal{N}}

\newcommand{\vI}{\mathbf{I}}

\newcommand{\vz}{\mathbf{z}}

\newcommand{\bvY}{\bar{\mathbf{Y}}}
\newcommand{\vY}{\mathbf{Y}}
\newcommand{\wass}[3]{\mathcal{W}_{#1}\left(#2,#3\right)}

\newcommand{\vy}{\mathbf{y}}
\newcommand{\cH}{\mathcal{H}}
\newcommand{\tilg}{\tilde{g}}
\newcommand{\cA}{\mathcal{A}}
\newcommand{\cP}{\mathcal{P}}

\newcommand{\KL}[2]{\mathsf{KL}\left(#1\pmb{||}#2\right)}

\newcommand{\poly}{\mathsf{poly}}
\newcommand{\muhat}{\hat{\mu}}
\newcommand{\nuhat}{\hat{\nu}}
\newcommand{\mubar}{\bar{\mu}}
\newcommand{\vybar}{\bar{\vy}}
\newcommand{\vYbar}{\bar{\vY}}
\newcommand{\vxbar}{\bar{\vx}}

\newcommand{\dn}[1]{{\color{blue}DN:#1}}
\newcommand{\ad}[1]{{\color{red}AD:#1}}

\newcommand{\answerTODO}[1][]{\textcolor{red}{\bf [TODO]}}

\title{Provably Fast Finite Particle Variants of SVGD via Virtual Particle Stochastic Approximation}

\author{%
  Aniket Das \\ 
  Google Research\\
  Bangalore, India\\
  \texttt{ketd@google.com} \\
  \And
  Dheeraj Nagaraj \\ 
  Google Research\\
  Bangalore, India\\
  \texttt{dheerajnagaraj@google.com} \\
}

\begin{document}

\maketitle


\begin{abstract}

Stein Variational Gradient Descent (SVGD) is a popular variational inference algorithm which simulates an interacting particle system to approximately sample from a target distribution, with impressive empirical performance across various domains. Theoretically, its population (i.e, infinite-particle) limit dynamics is well studied but the behavior of SVGD in the finite-particle regime is much less understood. In this work, we design two computationally efficient variants of SVGD, namely \ouralg~(which is conceptually elegant) and \ouralgnew~(which is empirically effective), both of which exhibit provably fast finite-particle convergence rates. By introducing the concept of \emph{virtual particles}, we develop novel stochastic approximations of population-limit SVGD dynamics in the space of probability measures, which are exactly implementable using only a finite number of particles. Our algorithms can be viewed as specific random-batch approximations of SVGD, which are computationally more efficient than ordinary SVGD. We show that the $n$ particles output by \ouralg~and \ouralgnew, run for $T$ steps with batch-size $K$, are at-least as good as i.i.d samples from a distribution whose Kernel Stein Discrepancy to the target is at most $O(\nicefrac{d^{1/3}}{(KT)^{1/6}})$ under standard assumptions. Our results also hold under a mild growth condition on the potential function, which is much weaker than the isoperimetric (e.g. Poincare Inequality) or information-transport conditions (e.g. Talagrand's Inequality $\mathsf{T}_1$) generally considered in prior works. As a corollary, we consider the convergence of the empirical measure (of the particles output by \ouralg~and \ouralgnew) to the target distribution and demonstrate a \emph{double exponential improvement} over the best known finite-particle analysis of SVGD. Beyond this, our results present the \emph{first known oracle complexities for this setting with polynomial dimension dependence}, thereby completely eliminating the curse of dimensionality exhibited by previously known finite-particle rates.  
\end{abstract}

\section{Introduction}
Sampling from a distribution over $\bR^d$ whose density $\pistar(\vx) \propto \exp(-F(\vx))$ is known only upto a normalizing constant, is a fundamental problem in machine learning \citep{welling2011bayesian, ho2020denoising}, statistics \citep{roberts1996exponential, neal2011mcmc}, theoretical computer science \citep{lee2022manifold, gopi2022private} and statistical physics \citep{parisi1981correlation, el2022sampling}. A popular approach to this is the Stein Variational Gradient Descent (SVGD) algorithm introduced by \citet{liu2016stein}, which uses a positive definite kernel $k$ to evolve a set of $n$ interacting particles  $(\vx^{(i)}_t)_{i \in [n], t \in \mathbb{N}}$ as follows:
\begin{equation}
\label{eq:svgd-update}
\vspace{-0.05cm}
\vx^{(i)}_{t+1} \leftarrow \vx^{(i)}_t - \frac{\gamma}{n} \sum_{j=1}^{n} \left[k(\vx^{(i)}_t, \vx^{(j)}_t)\nabla F(\vx^{(j)}_t) - \nabla_2 k(\vx^{(i)}_t, \vx^{(j)}_t)\right]
\end{equation}
SVGD exhibits remarkable empirical performance in a variety of Bayesian inference, generative modeling and reinforcement learning tasks \cite{liu2016stein,wang2018stein,zhuo2018message,jaini2021learning,wang2016learning,liu2017steinpg,sun2023convergence,haarnoja2017reinforcement} and usually converges rapidly to the target density while using only a few particles, often outperforming Markov Chain Monte Carlo (MCMC) methods. However, in contrast to its wide practical applicability, theoretical analysis of the behavior SVGD is a relatively unexplored problem. Prior works on the analysis of SVGD \citep{korba2020non, duncan2019geometry,liu2017stein,salim2022convergence, chewi2020svgd} mainly consider the population limit, where the number of particles $n \to \infty$. These works assume that the initial distribution of the (infinite number of) particles has a finite KL divergence to the target $\pistar$ and subsequently, interpret the dynamics of population-limit SVGD as projected gradient descent updates for the KL divergence on the space of probability measures, equipped with the Wasserstein geometry. Under appropriate assumptions on the target density, one can then use the theory of Wasserstein Gradient Flows to establish non-asymptotic (in time) convergence of population-limit SVGD to $\pistar$ in the Kernel Stein Discrepancy (KSD) metric.  

While the framework of Wasserstein Gradient Flows suffices to explain the behavior of SVGD in the population limit, the same techniques are insufficient to effectively analyze SVGD in the finite-particle regime. This is primarily due to the fact that the empirical measure $\muhat^{(n)}$ of a finite number of particles does not admit a density with respect to the Lebesgue measure, and thus, its KL divergence to the target is always infinite (i.e. $\KL{\muhat^{(n)}}{\pistar} = \infty$). In such a setting, a direct analysis of the dynamics of finite-particle SVGD becomes prohibitively difficult due to complex inter-particle dependencies. To the best of our knowledge, the pioneering work of \citet{shi2022finite} is the only result that obtains an explicit convergence rate of finite-particle SVGD by tracking the deviation between the law of $n$-particle SVGD and that of its population-limit. To this end, \citet{shi2022finite} show that for subgaussian target densities, the empirical measure of $n$-particle SVGD converges to $\pistar$ at a rate of $O(\sqrt{\tfrac{\mathsf{poly}(d)}{\log \log n^{\Theta(1/d)}}})$ in KSD \footnote{We explicate the dimension dependence in \citet{shi2022finite} by closely following their analysis}. The obtained convergence rate is quite slow and fails to adequately explain the impressive practical performance of SVGD.

Our work takes a starkly different approach to this problem and deliberately deviates from tracking population-limit SVGD using a finite number of particles. Instead, we directly analyze the dynamics of KL divergence along a carefully constructed trajectory in the space of distributions. To this end, our proposed algorithm, Virtual Particle SVGD (\ouralg) devises an \emph{unbiased stochastic approximation (in the space of measures) of the population-limit dynamics of SVGD}. We achieve this by considering additional particles called \emph{virtual particles} \footnote{(roughly) analogous to virtual particles in quantum field theory that enable interactions between real particles} which evolve in time but aren't part of the output (i.e. \emph{real particles}). These virtual particles are used only to compute information about the current population-level distribution of the real particles, and enable exact implementation of our stochastic approximation to population-limit SVGD, while using only a finite number of particles. 

Our analysis is similar in spirit to non-asymptotic analyses of stochastic gradient descent (SGD) that generally do not attempt to track gradient descent (analogous to population-limit SVGD in this case), but instead directly track the evolution of the objective function along the SGD trajectory using appropriate stochastic descent lemmas \citep{khaled2020better,jain2021making,das2022sampling}. The key feature of our proposed stochastic approximation is the fact that it can be implemented using only a finite number of particles. This allows us to design faster variants of SVGD with provably fast finite-particle convergence. 
\subsection{Contributions}
\begin{table}[t]
\centering
\resizebox{1.02\linewidth}{!}{
    \begin{tabular}{|l|c|c|c|c|}
        \hline
        \textbf{Result} & \textbf{Algorithm} & \textbf{Assumption}  & \textbf{Rate }& \textbf{\shortstack{Oracle \\ Complexity}} \\
        [0.5ex]
        \hline
        \citet{korba2020non} & \shortstack{ Population Limit \\ SVGD} & {\color{red} \shortstack{Uniformly Bounded \\ $\KSD{\pistar}{\mubar_t}{\pistar}$}} & $\tfrac{\poly(d)}{\sqrt{T}}$  & {\color{red}\shortstack{Not \\ Implementable}} \\
        [0.5ex]
        \hline
        \citet{salim2022convergence} & \shortstack{ Population Limit \\ SVGD} & Sub-gaussian $\pistar$ & $\tfrac{d^{\nicefrac{3}{2}}}{\sqrt{T}}$  & {\color{red}\shortstack{Not \\ Implementable}} \\
        [0.5ex]
        \hline
        \citet{shi2022finite} & SVGD & Sub-gaussian $\pistar$ & $\tfrac{\mathsf{poly}(d)}{\sqrt{\log \log n^{\Theta(\nicefrac{1}{d})}}}$  & $\tfrac{\poly(d)}{\epsilon^2} e^{{\Theta(d e^{\nicefrac{\mathsf{poly}(d)}{\epsilon^2}})}}$ \\
        [0.5ex]
        \hline
        \textbf{Ours, Corollary \ref{cor:vpsvgd-finite-particle-improved}} & \textbf{VP-SVGD} & \textbf{Sub-gaussian} $\pistar$ & $(\nicefrac{d}{n})^{\nicefrac{1}{4}} + (\nicefrac{d}{n})^{\nicefrac{1}{2}}$ & $\nicefrac{d^4}{\epsilon^{12}}$ \\
        [0.5ex]
        \hline
        \textbf{Ours, Corollary \ref{cor:gbsvgd-finite-particle-improved}} & \textbf{GB-SVGD} & \textbf{Sub-gaussian} $\pistar$ & $\nicefrac{d^{\nicefrac{1}{3}}}{n^{\nicefrac{1}{12}}} + (\nicefrac{d}{n})^{\nicefrac{1}{2}}$ & $\nicefrac{d^6}{\epsilon^{18}}$ \\
        [0.5ex]
        \hline
        \textbf{Ours, Corollary \ref{cor:vpsvgd-finite-particle-improved}} & \textbf{VP-SVGD} & { \color{OliveGreen} \textbf{Sub-exponential} $\pistar$} & $\tfrac{d^{\nicefrac{1}{3}}}{n^{\nicefrac{1}{6}}} + \tfrac{d}{n^{\nicefrac{1}{2}}}$ & $\nicefrac{d^6}{\epsilon^{16}}$ \\
        [0.5ex]
        \hline
        \textbf{Ours, Corollary \ref{cor:gbsvgd-finite-particle-improved}} & \textbf{GB-SVGD} & { \color{OliveGreen} \textbf{Sub-exponential} $\pistar$} & $\tfrac{d^{\nicefrac{3}{8}}}{n^{\nicefrac{1}{16}}} + \tfrac{d}{n^{\nicefrac{1}{2}}}$ & $\nicefrac{d^9}{\epsilon^{24}}$ \\
        [0.5ex]
        \hline
    \end{tabular}
}
    \vspace{2mm}
    \caption{Comparison of our results with prior works. $d$, $T$, and $n$ denote the dimension, no. of iterations and no. of output particles respectively. Oracle Complexity denotes number of evaluations of $\nabla F$ needed to achieve $\KSD{\pistar}{\cdot}{\pistar} \leq \epsilon$ and Rate denotes convergence rate w.r.t KSD metric. Note that: 1. Population Limit SVGD is not implementable as it requires infinite particles 2.  The uniformly bounded $\KSD{\pistar}{\mubar_t}{\pistar}$ assumption cannot be verified apriori and is much stronger than subgaussianity (see \cite{salim2022convergence} Lemma C.1)}
\label{tab:results_compare}
\end{table}
\textbf{\ouralg~and \ouralgnew} We propose two variants of SVGD that enjoy provably fast finite-particle convergence guarantees to the target distribution: Virtual Particle SVGD (\ouralg~in Algorithm~\ref{alg:vp_svgd_simple}) and Global Batch SVGD (\ouralgnew~in Algorithm~\ref{alg:rb_svgd}). \ouralg~is a conceptually elegant stochastic approximation (in the space of probability measures) of population-limit SVGD, and \ouralgnew~is a practically efficient version of SVGD which achieves good empirical performance. Our analysis of \ouralgnew~builds upon that of \ouralg . When the potential $F$ is smooth and satisfies a quadratic growth condition (which holds under subgaussianity of $\pistar$, a common assumption in prior works \citep{salim2022convergence, shi2022finite}), we show that the $n$ particles output by $T$ steps of our algorithms, run with batch-size $K$, are at least as good as i.i.d draws from a distribution whose Kernel Stein Discrepancy to $\pistar$ is at most $O(\nicefrac{d^{1/3}}{(KT)^{1/6}})$. Our results also hold under a mild subquadratic growth condition for $F$, which is much weaker than isoperimetric (e.g. Poincare Inequality) or information-transport (e.g. Talagrand's Inequality $\mathsf{T}_1$) assumptions generally considered in the sampling literature \citep{vempala2019rapid, salim2022convergence, shi2022finite, sinho-lmc-poincare-lsi, sinho-non-log-concave-lmc}.

\textbf{State-of-the-art Finite Particle Guarantees} As corollaries of the above result, we establish that \emph{\ouralg~and \ouralgnew~exhibit the best known finite-particle guarantees in the literature which significantly outperform that of prior works}. Our results are summarized in Table \ref{tab:results_compare}. In particular, under subgaussianity of the target distribution $\pistar$, we show that the empirical measure of the $n$ particles output by \ouralg~converges to $\pistar$ in KSD at a rate of $O((\nicefrac{d}{n})^{\nicefrac{1}{4}} + (\nicefrac{d}{n})^{\nicefrac{1}{2}})$. Similarly, the empirical measure of the $n$ output particles of \ouralgnew~converges to $\pistar$ at a KSD rate of $O(\nicefrac{d^{\nicefrac{1}{3}}}{n^{\nicefrac{1}{12}}} + (\nicefrac{d}{n})^{\nicefrac{1}{2}})$. Both these results represent a \textbf{double exponential improvement} over the $O(\tfrac{\poly(d)}{\sqrt{\log \log n^{\Theta(\nicefrac{1}{d})}}})$ KSD rate of $n$-particle SVGD obtained by \citet{shi2022finite}, which, to our knowledge, is the best known prior result for SVGD in the finite particle regime. When benchmarked in terms of gradient oracle complexity, i.e., the number of evaluations of $\nabla F$ required by an algorithm to 
achieve $\KSD{\pistar}{\cdot}{\pistar} \leq \epsilon$, we demonstrate that for subgaussian $\pistar$, the oracle complexity of \ouralg~is $O(\nicefrac{d^4}{\epsilon^{12}})$ while that of \ouralgnew~is $O(\nicefrac{d^6}{\epsilon^{18}})$. To the best of our knowledge, our result presents the \emph{first known oracle complexity guarantee with polynomial dimension dependence}, and consequently, does not suffer from a curse of dimensionality unlike prior works. Furthermore, as discused above, the conditions under which our result holds is far weaker than subgaussianity of $\pistar$, and as such, includes sub-exponential targets and beyond. In particular, \emph{our guarantees for sub-exponential target distributions are (to the best of our knowledge) the first of its kind.} 


\textbf{Computational Benefits:} \ouralg~and \ouralgnew~can be viewed as specific random batch approximations of SVGD. Our experiments (Section~\ref{sec:exps}) show that \ouralgnew~obtains similar performance as SVGD but requires fewer computations. In this context, a different kind of random batch method that divides the particles into random subsets of interacting particles, has been proposed by \citet{li2020stochastic}. However, the objective in \citet{li2020stochastic} is to approximate finite-particle SVGD dynamics using the random batch method, instead of analyzing convergence of the random batch method itself. As such, their  guarantees also suffer from an exponential dependence on the time $T$. As explained below, their approach is also conceptually different from our method since we use the \textit{same} random batch to evolve \textit{every} particle, allowing us to interpret this as a stochastic approximation in the space of distributions instead of in the path space.

\subsection{Technical Challenges}
We resolve the following important conceptual challenges, which may be of independent interest. 

\textbf{Stochastic Approximation in the Space of Probability Measures} Stochastic approximations are widely used in optimization, control and sampling \citep{kushner2012stochastic, welling2011bayesian, jin2020random}. In the context of sampling, stochastic approximations are generally implemented in path space, e.g., Stochastic Gradient Langevin Dynamics (SGLD) \citep{welling2011bayesian} takes a random batch approximation of the drift term via the update $\vx_{t+1} = \vx_t - \tfrac{\eta}{K} \sum_{j =0}^{K-1} \nabla f(\vx_t, \xi_j) + \sqrt{2 \eta} \epsilon_t, \ \epsilon_t \sim \cN(0, \vI)$ where $\bE[f(\vx_t, \xi_j) | \vx_t] = F(\vx_t)$. Such stochastic approximations are then analyzed using the theory of stochastic processes over $\bR^d$ \citep{das2022utilising, raginsky2017non, zou2021faster, kinoshita2022improved}. However, when viewed in the space of probability measures (i.e, $\mu_t = \mathrm{Law}(\vx_t)$), the time-evolution of these algorithms is deterministic. In contrast, our approach designs \emph{stochastic approximations in the space of probability measures}. In particular, the time-evolution of the law of any particle in \ouralg~and \ouralgnew~are a stochastic approximation of the dynamics of \textit{population-limit} SVGD. We ensure that requires only a finite number of particles for exact implementation. 

\textbf{Tracking KL Divergence in the Finite-Particle Regime} The population limit ($n \to \infty$) ensures that the initial empirical distribution ($\mu_0$) of SVGD admits a density (w.r.t the Lebesgue measure). When the KL divergence between $\mu_0$ and $\pistar$ is finite, prior works on population-limit SVGD analyze the time-evolution of the KL divergence to $\pistar$. However, this approach cannot be directly used to analyze finite-particle SVGD since the empirical distribution of a finite number of particles does not admit a density, and thus its KL divergence to $\pistar$ is infinite. Our analysis of \ouralg~and \ouralgnew~circumvents this obstacle by considering the dynamics of an infinite number of particles, whose empirical measure then admits a density. However, the careful design ensures that the dynamics of $n$ of these particles can be computed exactly, using only a finite total number of (real + virtual) particles. When conditioned on the virtual particles, these particles are i.i.d. and their conditional law is close to the target distribution with high probability.

\section{Notation and Problem Setup}
\label{sec:notation}
We use $\|\cdot\|, \langle \cdot ,\cdot \rangle$ to denote the Euclidean norm and inner product over $\bR^d$ respectively. All other norms and inner products are subscripted to indicate their underlying space. $\cP_2(\bR^d)$ denotes the space of probability measures on $\bR^d$ with finite second moment, with the Wasserstein-2 metric denoted as $\wass{2}{\mu}{\nu}$ for $\mu, \nu \in \cP_2(\bR^d)$. For any two probability measures $\mu,\nu$, we denote their KL divergence as $\KL{\mu}{\nu}$. For any function $f: X\to Y$ and any probability measure $\mu$ over $X$, we let $f_{\#}\mu$ denote the law of $f(\vx): \vx \sim \mu$. Given a sigma algebra $\mathcal{F}$ over some space $\Omega$, and a measurable space $X$, $\mu(\cdot \ ;\cdot): \mathcal{F} \times \mathcal{X} \to \bR^{+}$ is a probability kernel if for every $x \in \mathcal{X}$, $\mu(\cdot\ ;x)$ is a measure over $\mathcal{F}$ and for every $A \in \mathcal{F}$, the map $x \to \mu(A\ ; x)$ is measurable. We make use of probability measures $\mu(\cdot \ ; \vx)$ where $\vx$ is a random element of some appropriate space $\mathcal{X}$, resulting in random probability measures. We use $[m]$ and $(m)$ to denote the sets $\{1, \dots, m \}$ and $\{0, \dots, m-1\}$ respectively. For any finite set $A$, $\mathbb{S}_A$ denotes the group of all permutations of $A$. We use the $O$ notation to characterize the dependence of our rates on the number of iterations $T$, dimension $d$ and batch-size $K$, suppressing numerical and problem-dependent constants. We use $\lesssim$ to denote $\leq$ upto universal constants. 

We fix a symmetric positive definite reproducing kernel $k : \bR^d \times \bR^d \to \bR$ and let the corresponding reproducing kernel Hilbert space (RKHS) \citep{steinwart2008support} be denoted as $\mathcal{H}_0$. We denote the product RKHS as $\mathcal{H} = \prod_{i=1}^{d}\mathcal{H}_0$, equipped with the standard inner product for product spaces. We assume $k$ is differentiable in both its arguments and let $\nabla_2 k(\vx,\vy)$ denote the gradient of $k(\cdot,\cdot)$ with respect to the second argument. For any $\mu \in \cP_2(\bR^d)$, we assume $\cH \subset L^2(\mu)$ and the inclusion map $i_\mu : \cH \to L^2(\mu)$ is continuous. We use $P_\mu : L^2(\mu) \to \cH$ to denote the adjoint of $i_\mu$, i.e., the unique operator which satisfies $\dotp{f}{i_{\mu} g}_{L^2(\mu)} = \dotp{P_\mu f}{g}_{\cH}$ for any $f \in L^2(\mu), g \in \cH$. \citet{carmeli2010vector} shows that $P_\mu$ can be expressed as a kernel convolution, i.e., $(P_\mu f)(\vx) = \int k(\vx, \vy) f(\vy) \mathsf{d} \mu(\vy)$

We define the function $h : \bR^d \times \bR^d \to \bR$ as $h(\vx, \vy) = k(\vx, \vy) \nabla F(\vy) - \nabla_2 k(\vx, \vy)$ and $h_\mu \in \cH$ as $h_\mu = P_\mu ( \nabla_{\vx} \log(\tfrac{\mathsf{d} \mu}{\mathsf{d} \pistar}(\vx)))$ for any $\mu \in \cP_2(\bR^d)$. Integration by parts shows that $h_\mu(\vx) = \int h(\vx, \vy) \mathsf{d} \mu(\vy)$. The convergence metric we use is the Kernel Stein Discrepancy (KSD) metric, which is widely used for comparing probability distributions \citep{liu2016kernelized,chwialkowski2016kernel} and analyzing SVGD  \citep{salim2022convergence, korba2020non}.
\begin{definition}[Kernel Stein Discrepancy]
\label{def:ksd}
Define the Langevin Stein Operator of $\pistar$ acting on any differentiable $g : \bR^d \to \bR^d$:
$$(T_{\pistar} g)(\vx) = \nabla \cdot g(\vx) - \dotp{\nabla F(\vx)}{g(\vx)}$$
For any two probability measures $\mu, \nu$, the Kernel Stein Discrepancy between $\mu$ and $\nu$ (with respect to $\pistar$), denoted as $\KSD{\pistar}{\mu}{\nu}$ is defined as
$$\KSD{\pistar}{\mu}{\nu} = \sup_{\|g\|_{\cH} \leq 1} \bE_{\mu}[T_{\pistar}g] - \bE_{\nu}[T_{\pistar}g]$$
Using integration by parts (see \citet{chwialkowski2016kernel}), it follows that $\KSD{\pistar}{\mu}{\nu} = \|h_{\mu} - h_\nu \|_{\cH}$ 
\end{definition}
\paragraph{Organization:} We review the population-limit SVGD in Section \ref{sec:pop-svgd} and derive \ouralg~and \ouralgnew~in Section \ref{sec:alg}. We state our technical assumptions in Section \ref{sec:assumptions}, main results in Section \ref{sec:results} and provide a proof sketch in Section \ref{sec:pf-sketch}. We present empirical evaluation in Section \ref{sec:exps}.
\section{Background on Population-Limit SVGD}
\label{sec:pop-svgd}
We briefly introduce the analysis of population-limit SVGD using the theory of Wasserstein Gradient Flows and refer the readers to \citet{korba2020non} and \citet{salim2022convergence} for a detailed treatment. 

The space $\cP_2(\bR^d)$ equipped with the 2-Wasserstein metric $\mathcal{W}_2$ is known as the Wasserstein space, which admits the following Riemannian structure : For any $\mu \in \cP_2(\bR^d)$, the tangent space $T_{\mu} \cP_2(\bR^d)$ can be identified with the Hilbert space $L^2(\mu)$. We can then define differentiable functionals $\mathcal{L} : \cP_2(\bR^d) \to \bR$ and compute their Wasserstein gradients, denoted as $\nabla_{\mathcal{W}_2} \mathcal{L}$. Note that the target $\pistar$ is the unique minimizer over $\cP_2(\bR^d)$ for the functional $\mathcal{L}[\mu] = \KL{\mu}{\pistar}$. The Wasserstein Gradient of $\mathcal{L}[\mu]$ is $\nabla_{\mathcal{W}_2} \mathcal{L}[\mu] = \nabla_{\vx} \log(\tfrac{\mathsf{d} \mu}{\mathsf{d} \pistar}(\vx))$  \cite{ambrosio2005gradient}. This powerful machinery has served as a backbone for the analysis of algorithms such as LMC \citep{wibisono2018-sampling-as-opt, bernton2018-lmc-jko, sinho-non-log-concave-lmc} and population-limit SVGD \citep{duncan2019geometry, korba2020non, salim2022convergence, sun2023convergence, chewi2020svgd}.

The updates of population-limit SVGD can be viewed as Projected Gradient Descent in the Wasserstein space. Recall from Section \ref{sec:notation} that the function $h_\mu(\vx) = P_\mu (\nabla \log(\tfrac{\mathsf{d} \mu}{\mathsf{d} \pistar}))(\vx) = \int h(\vx, \vy) \mathsf{d} \mu(\vy)$. Let $\muhat^{n}_t$ denote the empirical measures of the SVGD particles $(\vx^{(i)}_t)_{i \in [n]}$ at timestep $t$. We note that the SVGD updates in \eqref{eq:svgd-update} can be recast as $\muhat^n_{t+1} = (I - \gamma h_{\muhat^n_t})_{\#}\muhat^t_n$. In the limit of infinite particles $n \to \infty$, suppose the empirical measure $\muhat^n_t$ converges to the population measure $\mubar_t$. In this population limit, the updates of SVGD can be expressed as,
\begin{align*}
    \mubar_{t+1} = \left(I - h_{\mubar_t}\right)_{\#}\mubar_t = \left(I - \gamma P_{\mubar_t}\left(\nabla \log(\tfrac{d \mubar_t}{d \pistar})\right)\right)_{\#}\mubar_t = \left(I - \gamma P_{\mubar_t}\left(\nabla_{\mathcal{W}_2} \KL{\mubar_t}{\pistar} \right)\right)_{\#}\mubar_t
\end{align*}
Recall from Section \ref{sec:notation} that $P_{\mubar_t} : L^2(\mubar_t) \to \cH$ is the Hilbert adjoint of $i_{\mubar_t}$. Since $\cH \subset L^2(\mubar_t)$, the updates of SVGD in the population limit can be seen as Projected Wasserstein Gradient Descent for $\mathcal{L}[\mu] = \KL{\mu}{\pistar}$, with the Wasserstein Gradient at each step being projected onto the RKHS $\cH$. Assuming $\KL{\mubar_0}{\pistar} < \infty$, convergence of population limit SVGD is then established by tracking the evolution of $\KL{\mubar_t}{\pistar}$ under appropriate structural assumptions (such as subgaussianity) on $\pistar$.
\section{Algorithm and Intuition}
\label{sec:alg}
In this section, we derive \ouralg~(Algorithm~\ref{alg:vp_svgd_simple}), and build upon it to obtain \ouralgnew. Consider a countably infinite collection of particles $\vx_0^{(l)} \in \bR^{d}, \ l \in \mathbb{N} \cup \{0\}$,  sampled i.i.d from a measure $\mu_0$, having a density w.r.t. the Lebesgue measure. By the strong law of large numbers, the empirical measure of $\vx_0^{(l)}$ is almost surely equal to $\mu_0$ (see \citet[Theorem 11.4.1]{dudley2018real}). Let batch size $K \in \bN$ denote the batch size, and $\cF_t$ denote the filtration $\mathcal{F}_t, \ t \geq 0$ as $\cF_t = \sigma(\{ \vx_0^{(l)} \ | \ l \leq Kt-1  \}), \ \forall \ t \in \mathbb{N}$, with $\cF_0$ being the trivial $\sigma$ algebra. For ease of exposition, we discuss the case of $K=1$ in this section and present a complete derivation for arbitrary $K \geq 1$ in Section \ref{app-sec:vpsvgd-analysis}. Recall from Section \ref{sec:pop-svgd} that the updates of population-limit SVGD in $\cP_2(\bR^d)$ can be expressed as follows:
\begin{equation}\label{eq:pop_svgd}
    \bar{\mu}_{t+1} = (I-\gamma h_{\bar{\mu}_t})_{\#} \bar{\mu}_t
\end{equation}
We aim to design a stochastic approximation in $\cP_2(\bR^d)$ for the updates \eqref{eq:pop_svgd}, such that it admits a finite-particle realization. To this end, we propose the following dynamics in $\bR^d$
\begin{equation}\label{eq:diagonal_dynamics}
    \vx_{t+1}^{(s)} = \vx_t^{(s)} - \gamma h(\vx_t^{(s)},\vx_t^{(t)}), \quad s \in \mathbb{N} \cup \{0\}
\end{equation}
Now, for each time-step $t$, we focus on the time evolution of the particles $(\vx^{(l)}_t)_{l \geq t}$ (called the \emph{lower triangular evolution}). From~\eqref{eq:diagonal_dynamics}, we observe that for any $t \in \mathbb{N}$ and $l \geq t$, $\vx^{(l)}_{t}$ depends only on $\vx^{(0)}_0, \dots, \vx^{(t-1)}_0, \vx^{(l)}_0$. Therefore there exists a deterministic, measurable function $H_{t}$ such that:
\begin{align}
\label{eq:dependence-eqn}
    \vx^{(l)}_t = H_{t}(\vx^{(0)}_0, \dots, \vx^{(t-1)}_0, \vx^{(l)}_0)  \,; \quad \text{for every } l \geq t
\end{align}
Since $\vx^{(0)}_0, \dots, \vx^{(t-1)}_0, \vx^{(l)}_0 \iidsim \mu_0$, we conclude from \eqref{eq:dependence-eqn} that $(\vx^{(l)}_t)_{l \geq t}$ are i.i.d when  conditioned on $\vx^{(0)}_0, \dots, \vx^{(t-1)}_0$. To this end, we define the random measure $\mu_t | \cF_t$ as the law of $\vx_t^{(t)}$ conditioned on $\mathcal{F}_t$, i.e., $\mu_t|\cF_t$ is a probability kernel $ \mu_t( \cdot\ ; \vx^{(0)}_0, \dots, \vx^{(t-1)}_0)$, where $\mu_0 | \cF_0 := \mu_0$. By the strong law of large numbers, $\mu_t | \cF_t$ is equal to the empirical measure of $(\vx^{(l)}_t)_{l \geq t}$ conditioned on $\cF_t$. We will use $\mu_t|\cF_t$ and $\mu_t( \cdot\ ; \vx^{(0)}_0, \dots, \vx^{(t-1)}_0)$ interchangeably.  

Define the random function $g_t : \bR^d \to \bR^d$ as $g_t(\vx) := h(\vx, \vx^{(t)}_{t})$. From \eqref{eq:dependence-eqn}, we note that $g_t$ is $\cF_{t+1}$ measurable. From \eqref{eq:diagonal_dynamics}, we infer that the particles satisfy the following relation:
\begin{align*}
    \vx^{(s)}_{t+1} = (I - \gamma g_t)(\vx^{(s)}_{t}), \quad s \geq t+1
\end{align*}
Recall that $\vx_{t+1}^{(s)} | \vx^{(0)}_0, \dots, \vx^{(t)}_0 \sim \mu_{t+1} | \cF_{t+1}$ for any $s \geq t+1$. Furthermore, from  Equation~\eqref{eq:dependence-eqn}, we note that for $s \geq t+1$, $\vx^{(s)}_t$ depends only on $\vx^{(0)}_0, \dots, \vx^{(t-1)}_0$ and $\vx^{(s)}_0$. Hence, we conclude that $\mathrm{Law}(\vx^{(s)}_t | \vx^{(0)}_0, \dots, \vx^{(t)}_0) = \mathrm{Law}(\vx^{(s)}_t | \vx^{(0)}_0, \dots, \vx^{(t-1)}_0) = \mu_t | \cF_t$. With this insight, the dynamics of the lower-triangular evolution in $\cP_2(\bR^d)$ that the following holds almost surely:
\begin{equation}
\label{eq:diagonal-dynamics-pop-limit}
\mu_{t+1} | \cF_{t+1} = (I-\gamma g_{t})_{\#}\mu_t | \cF_t
\end{equation} 
$\vx^{(t)}_t | \cF_t \sim \mu_t | \cF_t$ implies $\bE[g_t(\vx) | \cF_t] = h_{\mu_t | \cF_t}(\vx)$. Thus \emph{lower triangular dynamics} \eqref{eq:diagonal-dynamics-pop-limit} is a stochastic approximation in $\cP_2(\bR^d)$ to the population limit of SVGD \eqref{eq:pop_svgd}. Setting the batch size to general $K$ and tracking the evolution of the first $KT+n$ particles, we obtain \ouralg~(Algorithm \ref{alg:vp_svgd_simple}).
\begin{algorithm}[ht] 
\caption{Virtual Particle SVGD (\texttt{\ouralg})} \label{alg:vp_svgd_simple}
 \textbf{Input}:   Number of steps $T$, number of output particles $n$, batch size $K$, Initial positions $\vx_0^{(0)},\dots,\vx_{0}^{(n+KT-1)} \iidsim \mu_0$, Kernel $k$, step size $\gamma$.
 \begin{algorithmic}[1] 
 \FOR{$t \in \{0,\ldots,T-1\}$}
 \FOR{$s \in \{0,\ldots,KT+n-1\}$}
 \STATE $\vx_{t+1}^{(s)} = \vx_{t}^{(s)} - \frac{\gamma}{K}\sum_{l=0}^{K-1}[ k(\vx_t^{(s)},\vx_t^{(tK+l)})\nabla F(\vx_t^{(tK+l)}) - \nabla_{2} k(\vx_t^{(s)},\vx_t^{(tK+l)})]$ 
 \ENDFOR
 \ENDFOR
 \STATE Draw $S$ uniformly at random from $\{0,\dots,T-1\}$
 \STATE Output $(\vy^{(0)},\dots,\vy^{(n-1)})= (\vx_{S}^{(TK)},\dots, \vx^{(TK+n-1)}_{S})$
\end{algorithmic}
\end{algorithm}

\textbf{Virtual Particles} In Algorithm~\ref{alg:vp_svgd_simple}, $(\vx_t^{(l)})_{KT \leq l\leq KT+n-1}$ are the \emph{real particles} which constitute the output. $(\vx_{t}^{(l)})_{l < KT}$ are \emph{virtual particles} which propagate information about the probability measure $\mu_t | \cF_t$ to enable computation of $g_t$, an unbiased estimate of the projected Wasserstein gradient $h_{\mu_t | \cF_t}$.

\textbf{\ouralg~as SVGD Without Replacement} \ouralg~is a without-replacement random-batch approximation of SVGD \eqref{eq:svgd-update}, where a different batch is used across timesteps, but the same batch is across particles given a fixed timestep. With i.i.d. initialization, picking the `virtual particles' in a fixed order or from a random permutation does not change the evolution of the real particles. With this insight, we design \ouralgnew~(Algorithm~\ref{alg:rb_svgd}) where we consider $n$ particles \textit{and} output $n$ particles (instead of wasting $KT$ particles as `virtual particles') via a random-batch approximation of SVGD.

\begin{algorithm}[ht] 
\caption{Global Batch SVGD (\texttt{\ouralgnew})  } \label{alg:rb_svgd}
 \textbf{Input}:  $\#$ of time steps $T$, $\#$ of particles $n$, $\vx_0^{(0)},\dots,\vx_{0}^{(n-1)} \iidsim \mu_0$, Kernel $k$, step size $\gamma$, Batch size $K$, Sampling method $\in \{\text{with replacement}, \text{without replacement}\}$ 
 \begin{algorithmic}[1] 
 \FOR{$t \in \{0,\ldots,T-1\}$}
  \STATE $\mathcal{K}_t \leftarrow$ random subset of $[n]$ of size $K$ (via. sampling method)
 \FOR{$s \in \{0,\dots,n-1\}$}
 \STATE $\vx_{t+1}^{(s)} = \vx_{t}^{(s)} - \frac{\gamma}{K}\sum_{r \in \mathcal{K}_t}[ k(\vx_t^{(s)},\vx_t^{(r)})\nabla F(\vx_t^{(r)}) - \nabla_{2} k(\vx_t^{(s)},\vx_t^{(r)})]$ 
 \ENDFOR
 \ENDFOR
 \STATE Draw $S$ uniformly at random from $\{0,1,\dots,T-1\}$
 \STATE Output $(\bar{\vy}^{(0)},\dots,\bar{\vy}^{(n-1)}) = (\vx_{S}^{(0)},\dots, \vx^{(n-1)}_{S})$
\end{algorithmic}
\end{algorithm}
In Algorithm~\ref{alg:rb_svgd}, with replacement sampling means selecting a batch of $K$ particles i.i.d. from the uniform distribution over $[n]$. Without replacement sampling means fixing a random permutation $\sigma$ over $\{0,\dots,n-1\}$ and selecting the batches in the order specified by the permutation.

\section{Assumptions}
\label{sec:assumptions}
In this section, we discuss the key assumptions required for our analysis of \ouralg~and \ouralgnew. Our first assumption is smoothness of $F$, which is standard in optimization and sampling.
\begin{assumption}[L-Smoothness]\label{as:smooth}
$\nabla F$ exists and is $L$ Lipschitz. Moreover $\|\nabla F (0)\| \leq \sqrt{L} $. 
\end{assumption}
It is easy find a point such that $\|\nabla F(\vx^{*})\| \leq \sqrt{L}$ (e.g., using $\Theta(1)$ gradient descent steps \citep{nesterov1998introductory}) and center the initialization at $\mu_0$ at $\vx^{*}$. For clarity, we take $\vx^{*} = 0$ without loss of generality. We now impose the following growth condition on $F$. 
\begin{assumption}[Growth Condition]\label{as:growth}
There exist $\alpha, d_1,d_2 > 0$ such that $$F(\vx) \geq d_1 \|\vx\|^\alpha - d_2 \quad \forall \vx \in \bR^{d}$$
\end{assumption}
Note that Assumption \ref{as:smooth} ensures $\alpha \leq 2$. Assumption \ref{as:growth} is essentially a tail decay assumption on the target density $\pistar(\vx) \propto e^{-F(\vx)}$. In fact, as we shall show in Appendix \ref{app-sec:tech-lemmas}, Assumption \ref{as:smooth} ensures that the tails of $\pistar$ decay as $\propto e^{-\|\vx\|^\alpha}$. Consequently, Assumption \ref{as:growth} holds with $\alpha = 2$ when $\pistar$ is subgaussian and with $\alpha=1$ when $\pistar$ is subexponential. Subgaussianity is equivalent to $\pistar$ satisfying the $\mathsf{T}_1$ inequality \citep[Theorem 22.10]{villani2009optimal}, commonly assumed in prior works on SVGD \citep{salim2022convergence, shi2022finite}. We also note that subexponentiality is implied when $\pistar$ satisfies the Poincare Inequality \citep[Section 4]{bobkov1997poincare}, which is considered a mild condition in the sampling literature \citep{vempala2019rapid, sinho-lmc-poincare-lsi, sinho-non-log-concave-lmc, das2022utilising, chewi2020svgd}. This makes Assumption \ref{as:smooth} significantly weaker than the isoperimetric or information-transport assumptions considered in prior works. 

Next, we impose a mild assumption on the RKHS of the kernel $k$, which has been used by several prior works \citep{salim2022convergence, korba2020non, sun2023convergence, shi2022finite}.
\begin{assumption}[Bounded RKHS Norm]\label{as:norm}
For any $\vy \in \bR^d$, $k(\cdot, \vy)$ satisfies $\|k(\cdot,\vy)\|_{\mathcal{H}_0} \leq B$. Furthermore, $\nabla_2 k(\cdot, \vy) \in \cH$ and $\|\nabla_{2}k(\cdot,\vy)\|_{\mathcal{H}} \leq B$
\end{assumption}
Assumption \ref{as:norm} ensures that the adjoint operator $P_{\mu}$, used in Sections \ref{sec:notation} and \ref{sec:pop-svgd}, is well-defined. We also make the following assumptions on the kernel $k$, which is satisfied by a large class of standard kernels such as Radial Basis Function kernels and Mat\'ern kernels of order $\geq \nicefrac{3}{2}$. 
\begin{assumption}[Kernel Decay]\label{as:kern_decay}
The kernel $k$ satisfies the following for constants $A_1,A_2,A_3 > 0$.
\begin{align*}
    0 \leq k(\vx,\vy) \leq \tfrac{A_1}{1 + \|\vx -\vy\|^2}, \quad \ \
    \|\nabla_{2}k(\vx,\vy)\| \leq A_2, \quad \ \
    \|\nabla_{2}k(\vx,\vy)\|^2 \leq A_3 k(\vx,\vy) 
\end{align*}
\end{assumption}

Finally, we make the following mild assumption on the initialization.
\begin{assumption}[Initialization]\label{as:init}
 The initial distribution $\mu_0$ is such that $\KL{\mu_0}{\pistar} < \infty$. Furthermore, $\mu_0$ is supported in $\mathcal{B}(R)$, the $\ell_2$ ball of radius $R$
\end{assumption}
Since prior works usually assume Gaussian initialization \citep{salim2022convergence, vempala2019rapid}, Assumption \ref{as:init} may seem slightly non-standard. However, this is not a drawback. In fact, whenever $R = \Theta(\sqrt{d} + \mathsf{polylog}(\nicefrac{n}{\delta}))$, Gaussian initialization can be made indistinguishable from $\mathsf{Uniform}(\mathcal{B}(R))$ initialization, with probability at least $1 - \delta$, via a coupling argument. To this end, we impose Assumption \ref{as:init} for ease of exposition and our results can be extended to consider Gaussian initialization. In Appendix \ref{app-sec:tech-lemmas} we show that taking $R = \sqrt{\nicefrac{d}{L}}$ and $\mu_0 = \mathsf{Uniform}(\mathcal{B}(R))$ suffices to ensure $\KL{\mu_0}{\pistar} = O(d)$. 
\section{Results}
\label{sec:results}
\subsection{\ouralg}
Our first result, proved in Appendix \ref{app-sec:vpsvgd-analysis}, shows that the law of the \emph{real particles of} \ouralg~, when conditioned on the virtual particles, is close to $\pistar$ in KSD. As a consequence, it shows that the particles output by \ouralg~are i.i.d. samples from a random probability measure $\mubar(\cdot;\vx^{(0)}_0, \dots, \vx^{(KT-1)}_0, S)$ which is close to $\pistar$ in KSD. 
\begin{theorem}[\textbf{Convergence of~\ouralg}]
\label{thm:vpsvgd-bound}
Let $\mu_t$ be as defined in Section~\ref{sec:alg}. 
Let Assumptions~\ref{as:smooth}~\ref{as:growth},~\ref{as:norm},~\ref{as:kern_decay}, and~\ref{as:init} be satisfied and let $\gamma \leq \min\{\nicefrac{1}{2A_1L}, \nicefrac{1}{(4+L)B}\}$. There exist $(\zeta_i)_{0\leq i\leq 3}$ depending polynomially on $A_1,A_2,A_3,B,L,d_1,d_2$ for any fixed $\alpha \in (0, 2]$, such that whenever $\gamma \xi \leq \tfrac{1}{2 B}$, with $\xi = \zeta_0 + \zeta_1 (\gamma T)^{\nicefrac{1}{\alpha}} + \zeta_2 (\gamma^2 T)^{\nicefrac{1}{\alpha}} + \zeta_3 R^{\nicefrac{2}{\alpha}}$, the following holds:
\begin{align*}
    \frac{1}{T}\sum_{t=0}^{T-1}\bE\left[\KSDsq{\pistar}{\mu_t|\mathcal{F}_t}{\pistar}\right] \leq \frac{2\KL{\mu_{0}|\mathcal{F}_{0}}{\pistar}}{\gamma T} + \frac{\gamma B (4 + L) \xi^2 }{K}
\end{align*}
Define the probability kernel $\bar{\mu}(\cdot\ ;\cdot)$ as follows: For any $x_\tau \in \bR^d$, $\tau \in (KT)$ and $s\in (T)$, $\bar{\mu}(\cdot\ ;x_0,\dots,x_{KT-1},s) := \mu_s(\cdot\ ; x_0,\dots, x_{Ks-1})$ and $\bar{\mu}(\cdot\ ;x_0,\dots,x_{KT-1},s=0) := \mu_0(\cdot)$. Conditioned on $\vx^{(0)}_\tau = x_\tau,\ S = s$ for every $\tau \in (KT)$, the outputs $\vy^{(0)},\dots,\vy^{(n-1)}$ of \ouralg~are i.i.d samples from $\bar{\mu}(\cdot\ ;x_0,\dots,x_{KT-1},s)$. Furthermore,
\begin{align*}
    \bE[\KSDsq{\pistar}{\bar{\mu}(\cdot \ ;\vx^{(0)}_0,\dots,\vx_0^{(KT-1)},S)}{\pistar}] \leq \frac{2\KL{\mu_{0}|\mathcal{F}_{0}}{\pistar}}{\gamma T} + \frac{\gamma B (4 + L) \xi^2 }{K} 
\end{align*}
\end{theorem}
\textbf{Convergence Rates} Taking $\mu_0 = \mathsf{Uniform}(\mathcal{B}(R))$ with $R = \sqrt{\nicefrac{d}{L}}$ ensures $\KL{\mu_{0}|\mathcal{F}_{0}}{\pistar} = O(d)$ (see Appendix~\ref{app-sec:tech-lemmas}). Under this setting, choosing $\gamma = O(\tfrac{(Kd)^{\eta}}{T^{1 - \eta}})$ ensures that $\bE[\KSDsq{\pistar}{\mubar}{\pistar}] = O(\tfrac{d^{1-\eta}}{(KT)^{\eta}})$ where $\eta = \tfrac{\alpha}{2(1 + \alpha)}$. Thus, for $\alpha = 2$, (i.e, sub-Gaussian $\pistar$),  $ \mathsf{KSD}^2 = O(\tfrac{d^{\nicefrac{2}{3}}}{(KT)^{\nicefrac{1}{3}}})$. For $\alpha=1$ (i.e, sub-Exponential $\pistar$), the rate (in squared KSD) becomes $O(\tfrac{d^{\nicefrac{3}{4}}}{(KT)^{\nicefrac{1}{4}}})$. To the best of our knowledge, our convergence guarantee for sub-exponential $\pistar$ is the first of its kind.

\textbf{Comparison with Prior Works} \citet{salim2022convergence} analyzes population-limit SVGD for subgaussian $\pistar$, obtaining $\mathsf{KSD}^2 = O(\nicefrac{d^{\nicefrac{3}{2}}}{T})$ rate. We note that population-limit SVGD is not implementable whereas \ouralg~is an implementable algorithm whose outputs are samples from a distribution with guaranteed convergence to $\pistar$.

We also establish a convergence guarantee for \ouralg~which holds with high probability, when conditioned on the virtual particles $\vx^{(0)}_{0}, \dots, \vx^{(KT-1)}_0$. The proof of this result is presented in Appendix \ref{app-sec:vpsvgd-analysis}
\begin{theorem}[\textbf{\ouralg : High-Probability Rates}]
\label{thm:vpsvgd-hp-thm}
Let the assumptions and parameter settings of Theorem \ref{thm:vpsvgd-bound} apply and let $\delta \in (0, 1)$. Then, the following holds with probability at least $1 - \delta$:
\begin{align*}
    \frac{1}{T} \sum_{t=0}^{T-1} \KSD{\pistar}{\mu_t | \cF_t}{\pistar}^2 &\leq \frac{4 \KL{\mu_0 | \cF_0}{\pistar}}{\gamma T} + \frac{2 \gamma (4+L)B \xi^2}{K} \\
    &+ \frac{32 \xi^2 \log(\nicefrac{2}{\delta})}{KT} 
    +  12 \gamma (4+L) B \xi^2 \sqrt{\frac{\log(\nicefrac{2}{\delta})}{T}}
\end{align*}
Let $\mubar(\cdot; \vx^{(0)}_0, \dots, \vx^{(KT-1)}_0, S)$ be the probability kernel defined in the statement of Theorem \ref{thm:vpsvgd-bound}. Then, conditioned on $\vx^{(0)}_0, \dots, \vx^{(KT-1)}_0, S$, the $n$ particles output by \ouralg~are i.i.d samples from $\mubar(\cdot; \vx^{(0)}_0, \dots, \vx^{(KT-1)}_0, S)$. Furthermore, with probability at least $1 - \delta$
\begin{align*}
    \bE_S[\KSD{\pistar}{\mubar(\cdot; \vx^{(0)}_0, \dots, \vx^{(KT-1)}_0, S)}{\pistar}^2] &\leq \frac{4 \KL{\mu_0 | \cF_0}{\pistar}}{\gamma T} + \frac{2 \gamma (4+L)B \xi^2}{K} \\
    &+ \frac{32 \xi^2 \log(\nicefrac{2}{\delta})}{KT} 
    +  12 \gamma (4+L) B \xi^2 \sqrt{\frac{\log(\nicefrac{2}{\delta})}{T}}
\end{align*}
where $\bE_S$ denotes that the expectation is being taken only with respect to $S \sim \mathsf{Uniform}((T))$
\end{theorem}
\subsubsection{\ouralg~:Rapid Convergence of the Empirical Measure to the Target Distribution}
As a corollary of Theorem \ref{thm:vpsvgd-bound}, we show that the empirical distribution of the $n$ particles output by \ouralg~rapidly converges to $\pistar$ in KSD. The proof of this result is presented in Appendix \ref{prf:cor-vpsvgd-finite-particle-improved} 
\begin{corollary}[\textbf{\ouralg~: Fast Finite-Particle Convergence}]    
\label{cor:vpsvgd-finite-particle-improved}
Let the assumptions and parameter settings of Theorem \ref{thm:vpsvgd-bound} be satisfied. Let $\muhat^{(n)}$ denote the empirical measure of the $n$ particles output by VP-SVGD.
\begin{align*}
    \bE[\KSDsq{\pistar}{\muhat^{(n)}}{\pistar}] \leq \frac{\xi^2}{n} + \frac{2 \KL{\mu_0 | \cF_0}{\pistar}}{\gamma T} + \frac{\gamma B (4+L)\xi^2}{K}
\end{align*}
where $\xi$ is as defined in Theorem \ref{thm:vpsvgd-bound}. Setting $R = \sqrt{\nicefrac{d}{L}}, \gamma = O(\tfrac{(Kd)^{\eta}}{T^{1-\eta}})$ with $\eta = \tfrac{\alpha}{2(1+\alpha)}$ and $KT = d^{\tfrac{\alpha}{2+\alpha}} n^{\tfrac{2(1+\alpha)}{2+\alpha}}$ suffices to ensure,
\begin{align*}
    \bE[\KSDsq{\pistar}{\muhat^{(n)}}{\pistar}] \leq O\left(\frac{d^{\tfrac{2}{2+\alpha}}}{n^{\tfrac{\alpha}{2+\alpha}}} + \frac{d^{\nicefrac{2}{\alpha}}}{n}\right)
\end{align*}
\end{corollary}
\textbf{Comparison to Prior Work} For subgaussian $\pistar$ (i.e., $\alpha = 2$), \ouralg~exhibits a finite-particle rate of $\bE[\KSD{\pistar}{\muhat^{(n)}}{\pistar}] = O((\nicefrac{d}{n})^{\nicefrac{1}{4}} + (\nicefrac{d}{n})^{\nicefrac{1}{2}})$. This represents a \emph{double exponential improvement} over the SOTA $\Otilde(\tfrac{\mathsf{poly}(d)}{\sqrt{\log \log n^{\Theta(\nicefrac{1}{d})}}
})$ KSD rate obtained by \citet{shi2022finite} for SVGD for subgaussian targets. For subexponential $\pistar$ (i.e. $\alpha=1$), the KSD rate of \ouralg~is $O(\tfrac{d^{\nicefrac{1}{3}}}{n^{\nicefrac{1}{6}}} + \tfrac{d}{n^{\nicefrac{1}{2}}})$. To the best of our knowledge, this is the first known result of its kind.

\textbf{Oracle Complexity} As we shall show in Appendix \ref{app-sec:oracle-complexity-comparison}, the oracle complexity of VP-SVGD to achieve $\epsilon$-convergence in KSD for subgaussian $\pistar$ is $O(\nicefrac{d^4}{\epsilon^{12}})$, which is a double exponential improvement over the $O(\tfrac{\poly(d)}{\epsilon^2} e^{{\Theta(d e^{\nicefrac{\mathsf{poly}(d)}{\epsilon^2}})}})$ oracle complexity of SVGD as implied by \citet{shi2022finite}. Notably, the obtained oracle complexity of VP-SVGD does not suffer from a curse of dimensionality and to the best of our knowledge, is the \emph{first known oracle complexity for this problem with polynomial dimension dependence}. For subexponential $\pistar$, the obtained oracle complexity is $O(\nicefrac{d^6}{\epsilon^{16}})$ which is the first known result of its kind.

\subsection{\ouralgnew}
We now use \ouralg~as the basis to analyze \ouralgnew. Assume $n > KT$. Then, with probability at least $1 - \nicefrac{K^2 T^2}{n}$ (for with-replacement sampling) and $1$ (for without-replacement sampling), the random batches $\mathcal{K}_t$ in \ouralgnew~(Algorithm \ref{alg:rb_svgd}) are disjoint and contain distinct elements. When conditioned on this event $\cE$, we note that the $n-KT$ particles that were not included in any random batch $\mathcal{K}_t$ evolve exactly like the $n$ real particles of \ouralg. With this insight, we show that, conditioned on $\cE$, the outputs of \ouralg~and \ouralgnew~can be coupled such that the first $n-KT$ particles of both the algorithms are exactly equal. This allows us to derive the following squared KSD bound between the empirical measures of the outputs of \ouralg~and \ouralgnew. The proof of this result is presented in Appendix \ref{app-sec:gbsvgd-analysis}
\begin{theorem}[\textbf{KSD Bounds for~\ouralgnew}]
\label{thm:rb-svgd-bounds}
Let $n > KT$ and let $\vY = (\vy^{(0)}, \dots, \vy^{(n-1)})$ and $\vYbar = (\vybar^{(0)}, \dots, \vybar^{(n-1)})$ denote the outputs of \ouralg~and \ouralgnew~respectively. Moreover, let $\muhat^{(n)} = \tfrac{1}{n} \sum_{i = 0}^{n-1} \delta_{\vy^{(i)}}$ and $\nuhat^{(n)} = \tfrac{1}{n} \sum_{i = 0}^{n-1} \delta_{\vybar^{(i)}}$ denote their respective empirical measures. Under the assumptions and parameter settings of Theorem \ref{thm:vpsvgd-bound}, there exists a coupling of $\vY$ and $\vYbar$ such that the following holds:
\begin{equation}
    \bE[\KSDsq{\pistar}{\nuhat^{(n)}}{\muhat^{(n)}}] \leq \begin{cases}\frac{2 K^2T^2 \xi^2}{n^2} \quad &\text{(without replacement sampling)}\\
    \frac{2K^2T^2\xi^2}{n^2} \left(1 - \frac{K^2 T^2}{n}\right) + \frac{2K^2 T^2 \xi^2}{n}  \quad &\text{(with replacement sampling)}
    \end{cases}
\end{equation}
\end{theorem}
\subsubsection{\ouralgnew~:Rapid Convergence of the Empirical Measure to the Target Distribution}
Theorem \ref{thm:rb-svgd-bounds} ensures that the empirical distribution of the outputs of \ouralg~and \ouralgnew~are close to each other in squared KSD (under an appropriate coupling). To this end, we use Theorem \ref{thm:rb-svgd-bounds} and the finite-particle convergence guarantee of \ouralg~(i.e., Corollary \ref{cor:vpsvgd-finite-particle-improved}) to obtain the following finite-particle guarantee for \ouralgnew. We present a proof of this result in Appendix \ref{prf:cor-gbsvgd-finite-particle-improved}
\begin{corollary}[\textbf{\ouralgnew~: Fast Finite-Particle Convergence}]
\label{cor:gbsvgd-finite-particle-improved}
Let the assumptions and parameter settings of Theorem \ref{thm:vpsvgd-bound} be satisfied. Let $\nuhat^{(n)}$ denote the empirical measure of the $n$ particles output by GB-SVGD. Then, under without-replacement sampling of the minibatches, the following holds:
\begin{align*}
    \bE[\KSDsq{\pistar}{\nuhat^{(n)}}{\pistar}] \leq \frac{4 K^2 T^2 \xi^2}{n^2} + \frac{2\xi^2}{n} + \frac{4 \KL{\mu_0 | \cF_0}{\pistar}}{\gamma T} + \frac{2 \gamma B (4+L)\xi^2}{K}
\end{align*}
and the following holds under with-replacement sampling of the minibatches
\small
\begin{align*}
    \bE[\KSDsq{\pistar}{\nuhat^{(n)}}{\pistar}] \leq \frac{4 K^2 T^2 \xi^2}{n^2}(1 - \frac{K^2 T^2}{n}) + \frac{4 K^2 T^2 \xi^2}{n} + \frac{2\xi^2}{n} + \frac{4 \KL{\mu_0 | \cF_0}{\pistar}}{\gamma T} + \frac{2 \gamma B (4+L)\xi^2}{K}
\end{align*}
\normalsize
where $\xi$ is as defined in Theorem \ref{thm:vpsvgd-bound}. In particular, for GB-SVGD under without-replacement sampling of the minibatches, setting $R = \sqrt{\nicefrac{d}{L}}, \gamma = O(\tfrac{(Kd)^{\eta}}{T^{1-\eta}})$ with $\eta = \tfrac{\alpha}{2(1+\alpha)}$ and $KT = \sqrt{n}$ suffices to ensure the following 
\begin{align*}
    \bE[\KSDsq{\pistar}{\nuhat^{(n)}}{\pistar}] \leq O\left(\frac{d^{\nicefrac{2}{\alpha}}}{n} + \frac{d^{\tfrac{1}{1+\alpha}}}{n^{\tfrac{1+2\alpha}{2(1+\alpha)}}} + \frac{d^{\tfrac{2+\alpha}{2(1+\alpha)}}}{n^{\tfrac{\alpha}{4(1+\alpha)}}}\right)
\end{align*}
\end{corollary}
\textbf{Comparison to Prior Work} For subgaussian $\pistar$ (i.e., $\alpha = 2$), \ouralgnew~exhibits a finite-particle rate of $\bE[\KSD{\pistar}{\nuhat^{(n)}}{\pistar}] = O(\nicefrac{d^{\nicefrac{1}{3}}}{n^{\nicefrac{1}{12}}} + (\nicefrac{d}{n})^{\nicefrac{1}{2}})$. As before, this is a \emph{double exponential improvement} over the SOTA $\Otilde(\tfrac{\mathsf{poly}(d)}{\sqrt{\log \log n^{\Theta(\nicefrac{1}{d})}}
})$ KSD rate obtained by \citet{shi2022finite} for SVGD for subgaussian targets. For subexponential $\pistar$ (i.e. $\alpha=1$), the KSD rate of \ouralgnew~is $O(\tfrac{d^{\nicefrac{3}{8}}}{n^{\nicefrac{1}{16}}} + \tfrac{d}{n^{\nicefrac{1}{2}}})$, which, to the best of our knowledge, this is the first known result of its kind.

\textbf{Oracle Complexity} As we shall show in Appendix \ref{app-sec:oracle-complexity-comparison}, the oracle complexity of GB-SVGD to achieve $\epsilon$-convergence in KSD for subgaussian $\pistar$ is $O(\nicefrac{d^6}{\epsilon^{18}})$, which significantly improves upon the oracle complexity of SVGD as implied by \citet{shi2022finite}. Notably, the obtained oracle complexity of GB-SVGD does not suffer from a curse of dimensionality and to the best of our knowledge, is the first known oracle complexity for this problem with polynomial dimension dependence. For subexponential $\pistar$, the obtained oracle complexity is $O(\nicefrac{d^9}{\epsilon^{24}})$ which is the first known result of its kind.

\section{Proof Sketch}
\label{sec:pf-sketch}
We now present a sketch of our analysis. As shown in Section \ref{sec:alg}, the particles $(\vx^{(l)}_{t})_{l \geq Kt}$ are i.i.d conditioned on the filtration $\cF_t$, and the random measure $\mu_t | \cF_t$ is the law of $(\vx^{(Kt)}_{t})$ conditioned on $\vx^{(0)}_0, \dots, \vx^{(Kt-1)}_{0}$. Moreover, from equation \eqref{eq:diagonal-dynamics-pop-limit}, we know that $\mu_t | \cF_t$ is a stochastic approximation of population limit SVGD dynamics, i.e., $\mu_{t+1} | \cF_{t+1} = (I - \gamma g_t)_{\#}\mu_t | \cF_t$. Lemma~\ref{lem:evi-lemma} (similar to \citet[Proposition 3.1]{salim2022convergence} and \citet[Proposition 5]{korba2020non}) shows that under appropriate conditions, the KL between $\mu_t | \cF_t$ and $\pistar$ satisfies a (stochastic) descent lemma . Hence $\mu_t|\mathcal{F}_t$ admits a density and $\KL{\mu_t | \cF_t}{\pistar}$ is almost surely finite.
\begin{lemma}[Descent Lemma for $\mu_t | \cF_t$]
\label{lem:evi-lemma}
Let Assumptions \ref{as:smooth}, \ref{as:norm} and \ref{as:init} be satisfied and let $\beta > 1$ be an arbitrary constant. On the event $\gamma \|g_t\|_{\cH} \leq \tfrac{\beta - 1}{\beta B}$, the following holds almost surely
\begin{align*}
    \KL{\mu_{t+1} | \cF_{t+1}}{\pistar} \leq \KL{\mu_{t} | \cF_{t}}{\pistar} - \gamma \dotp{h_{\mu_t | \cF_t}}{g_t}_{\cH} + \frac{\gamma^2 (\beta^2 + L)B}{2} \|g_t\|^{2}_{\cH}
\end{align*}
\end{lemma}
 Lemma \ref{lem:evi-lemma} is analogous to the noisy descent lemma which is used in the analysis of SGD for smooth functions. Notice that $\mathbb{E}[g_t|\mathcal{F}_t] = h_{\mu_t|\mathcal{F}_t}$ (when interpreted as a Gelfand-Pettis integral \citep{talagrand1984pettis}, as discussed in Appendix \ref{app-sec:tech-lemmas} and Appendix \ref{app-sec:vpsvgd-analysis}) and hence in expectation, the KL divergence decreases in time. In order to apply Lemma \ref{lem:evi-lemma}, we establish an almost-sure bound on $\|g_t\|_{\cH}$ below.
\begin{lemma}
\label{lem:gt-bound-as}
Let Assumptions \ref{as:smooth}, \ref{as:growth}, \ref{as:norm}, \ref{as:kern_decay} and~\ref{as:init} hold. For $\gamma \leq \nicefrac{1}{2 A_1 L}$, the following holds almost surely,
\begin{align*}
    \|g_t\|_{\cH} \leq \xi = \zeta_0 + \zeta_1 (\gamma T)^{\nicefrac{1}{\alpha}} + \zeta_2 (\gamma^2 T)^{\nicefrac{1}{\alpha}}+ \zeta_3 R^{\nicefrac{2}{\alpha}}
\end{align*}
where $\zeta_0, \zeta_1, \zeta_2$ and $\zeta_3$ which depend polynomially on $A_1,A_2,A_3,B,d_1,d_2,L$ for any fixed $\alpha$.
\end{lemma}
 Let $K = 1$ for clarity. To prove Lemma~\ref{lem:gt-bound-as}, we first note via smoothness of $F(\cdot)$ and Assumption~\ref{as:norm} that $\|g_t\|_{\cH} \leq  C_0\|\vx_t^{(t)}\| + C_1$, and then bound $\|\vx_t^{(t)}\|$. Now, $g_s(\vx) = k(\vx,\vx_s^{(s)})\nabla F(\vx_s^{(s)}) - \nabla_2 k(\vx,\vx_s^{(s)})$. When $\|\vx_s^{(s)}-\vx\|$ is large, $\|g_s(\vx)\|$ is small due to decay assumptions on the kernel (Assumption~\ref{as:kern_decay}) implying that the particle does not move much. When $\vx_s^{(s)} \approx \vx$, we have $g_s(\vx) \approx k(\vx,\vx_s^{(s)})\nabla F(\vx) - \nabla_2 k(\vx,\vx_s^{(s)})$ and $k(\vx,\vx_s^{(s)}) \geq 0$. This is approximately a gradient descent update on $F(\cdot)$ along with a bounded term $\nabla_2 k(\vx,\vx_s^{(s)})$. Thus, the value of $F(\vx_t^{(l)})$ cannot grow too large after $T$ iterations. By Assumption~\ref{as:growth}, $F(\vx_t^{(l)})$ being small implies that $\|\vx_t^{(l)}\|$ is small. 

Equipped with Lemma \ref{lem:gt-bound-as}, we set the step-size $\gamma$ to ensure that the descent lemma (Lemma \ref{lem:evi-lemma}) always holds. The remainder of the proof involves unrolling through Lemma \ref{lem:evi-lemma} by taking iterated expectations on both sides. To this end we control $\dotp{h_{\mu_t | \cF_t}}{g_t}_{\cH}$ and $\|g_t\|^2_{\cH}$ in expectation, in Lemma~\ref{lem:pettis-exp}.
\begin{lemma}
\label{lem:pettis-exp}
Let Assumptions \ref{as:smooth},\ref{as:growth},\ref{as:norm},\ref{as:kern_decay},\ref{as:init} hold and $\xi$ be as defined in Lemma~\ref{lem:gt-bound-as}. Then, for $\gamma \leq \nicefrac{1}{2A_1 L}$,
$$\bE\left[\dotp{h_{\mu_t | \cF_t}}{g_t}_{\cH} | \cF_t\right] = \|h_{\mu_t | \cF_t}\|^2_{\cH} \text{\quad and \quad }\bE[\|g_t\|^2_{\cH}] \leq \nicefrac{\xi^2}{K} + \| h_{\mu_t | \cF_t} \|^2_\cH$$
\end{lemma}
\section{Experiments}
\label{sec:exps}
We compare the performance of \ouralgnew~and SVGD. We take $n = 100$ and use the Laplace kernel with $h = 1$ for both. We pick the stepsize $\gamma$ by a grid search for each algorithm. Additional details are presented in Appendix \ref{app-sec:exp-details}. We observe that SVGD takes fewer iterations to converge, but the compute time for \ouralgnew~is lower. This is similar to the typical behavior of stochastic optimization algorithms like SGD.  
\begin{figure}[h]
    \centering
    \begin{subfigure}[b]{0.45\textwidth}
    \includegraphics[width = \textwidth, height = 3.7 cm]{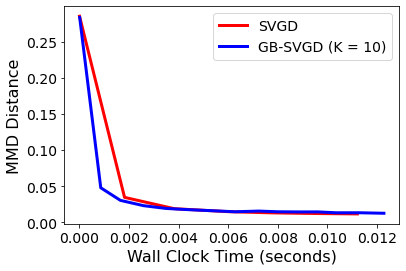}
    \caption{MMD vs Compute Time}
    \label{fig:gaussian_expt_time}
    \end{subfigure}
    \hfill
    \begin{subfigure}[b]{0.45\textwidth}
    \includegraphics[width = \textwidth, height = 3.7 cm]{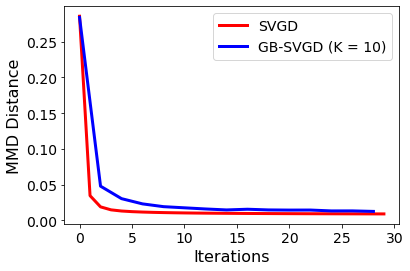}
    \caption{MMD vs Iterations}
    \label{fig:gaussian_expt_iters}
    \end{subfigure}
    \caption{Gaussian Experiment Comparing SVGD and \ouralgnew~averaged over 10 experiments.}
    \label{fig:gaussian_expt}
\end{figure}

\begin{figure}[th]
    \centering
    \begin{subfigure}[b]{0.45\textwidth}
    \includegraphics[width = \textwidth, height = 3.5 cm]{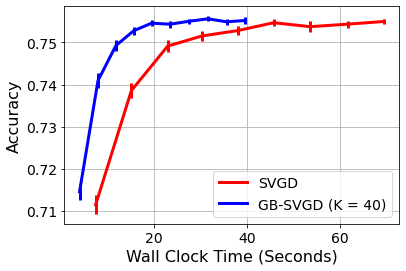}
    \caption{Accuracy vs Compute Time}
    \label{fig:covertype_time}
    \end{subfigure}
    \hfill
    \begin{subfigure}[b]{0.45\textwidth}
    \includegraphics[width = \textwidth, height = 3.5 cm ]{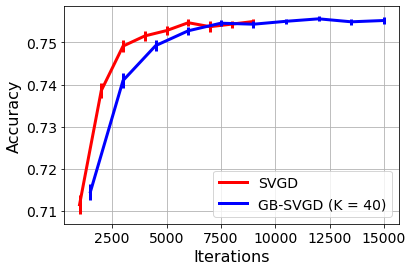}
    \caption{Accuracy vs Iterations}
    \label{fig:covertype_iters}
    \end{subfigure}
    \caption{Covertype Experiment, averaged over 50 runs. The error bars represent 95\% CI.}
        \label{fig:covertype_expt}
\end{figure}

\textbf{Sampling from Isotropic Gaussian (Figure~\ref{fig:gaussian_expt}):} As a sanity check, we set $\pistar = \cN(0, \vI)$ with $d = 5$. We pick $K = 10$ for \ouralgnew. The metric of convergence is MMD with respect to the empirical measure of 1000 i.i.d. sampled Gaussians.

\textbf{Bayesian Logistic Regression (Figure~\ref{fig:covertype_expt})} We consider the Covertype dataset which contains $\sim 580,000$ data points with $d = 54$. We consider the same priors suggested in \citet{gershman2012nonparametric} and implemented in \citet{liu2016stein}. We take $K=40$ for \ouralgnew. For both \ouralg~and \ouralgnew, we use AdaGrad with momentum to set the step-sizes as per \citet{liu2016stein}
\section{Conclusion}
\label{sec:conclusion}
We develop two computationally efficient variants of SVGD with provably fast convergence guarantees in the finite-particle regime, and present a wide range of improvements over prior work. A promising avenue of future work could be to establish convergence guarantees for SVGD with general non-logconcave targets, as was considered in recent works on LMC and SGLD \cite{sinho-non-log-concave-lmc, das2022utilising}. Other important avenues include establishing minimax lower bounds for SVGD and related particle-based variational inference algorithms. Beyond this, we also conjecture that the rates of \ouralgnew~can be improved even in the regime $n \ll KT$. However, we believe this requires new analytic tools. 
\section*{Acknowledgements}
We thank Jiaxin Shi, Lester Mackey and the anonymous reviewers for their helpful feedback. We are particularly grateful to Lester Mackey for providing insightful pointers on the properties of Kernel Stein Discrepancy, which greatly helped us in removing the curse of dimensionality from our finite-particle convergence guarantees. 
\bibliography{refs}
\newpage
\tableofcontents
\newpage
\appendix
\section{Additional Notation and Organization}
\label{app-sec:organization}
We use $\Gamma$ to denote the Gamma function $\Gamma(x) = \int_{0}^{\infty} t^{x-1} e^{-t} \dd t$, and recall that for any $n \in \mathbb{N}$, $\Gamma(n) = (n-1)!$. For any Lebesgue measurable $A \subseteq \bR^d$, we use $\vol(A)$ to denote it's Lebesgue Measure and $\Uniform(A)$ to denote the uniform distribution supported on $A$. We use $\cB(R)$ to denote the ball of radius $R$ centered at the origin, and recall that $\vol(\cB(R)) = \tfrac{\pi^{\nicefrac{d}{2}}}{\Gamma(\nicefrac{d}{2}+1)}R^d$. For ease of exposition, we assume $d \geq 2$. We further assume $\pistar(\vx) = e^{-F(\vx)}$. We note that this can be easily ensured by absorbing the normalizing constant into $F(0)$, and does not affect the dynamics of SVGD, \ouralg~or \ouralgnew~(since they only use the gradient information of $F$). We highlight that both these assumptions are made purely for the sake of clarity and are very easily removable with negligible changes to our analysis. 

In Appendix \ref{app-sec:tech-lemmas}, we discuss the technical lemmas used in our analysis, and present a short exposition to the Gelfand-Pettis integral in Appendix \ref{app-sec:gelfand-pettis}, which we use to analyze \ouralg. We analyze \ouralg~in Appendix \ref{app-sec:vpsvgd-analysis} and \ouralgnew~in Appendix \ref{app-sec:gbsvgd-analysis}. Convergence guarantees for the empirical measure of \ouralg~and \ouralgnew~are presented in Appendix \ref{app-sec:finite-particle-analysis}. We give a brief review of the related work in Section \ref{app-sec:more-lit-review}. Additional details regarding our experimental setup are stated in \ref{app-sec:exp-details}.
\section{Preliminaries}
\label{app-sec:tech-lemmas}
The following lemma shows that setting the initial distribution $\mu_0 = \mathsf{Uniform}(\mathcal{B}(R))$ with $R = \sqrt{\nicefrac{d}{L}}$ suffices to ensure $\KL{\mu_0}{\pistar} = O(d)$.The proof of this result is similar to that of \citet[Lemma 1]{vempala2019rapid} with the Gaussian initialization replaced by $\mathsf{Uniform}(\mathcal{B}(R))$ initialization. 
\begin{lemma}[\textbf{KL Upper Bound for Uniform Initialization}]
\label{lem:init_kl}
Let Assumption \ref{as:smooth} be satisfied and let $\mu_0 = \mathsf{Uniform}(\mathcal{B}(R))$ with $R = \sqrt{\nicefrac{d}{L}}$. Then, the following holds:
\begin{align*}
    \KL{\mu_0}{\pistar} \leq \frac{d}{2} \log(\nicefrac{L}{2 \pi}) + d + F(0) + \nicefrac{1}{2} \leq O(d)
\end{align*}
\end{lemma}
\begin{proof}
For any $\vx \in \bR^d$, the following holds by Assumption \ref{as:smooth}
\begin{align*}
    F(\vx) &\leq F(0) + \dotp{\nabla F(0)}{\vx} + \tfrac{L}{2}\|\vx\|^2 \\
    &\leq F(0) + \sqrt{L}\|\vx\| + \tfrac{L}{2}\|\vx\|^2 \\
    &\leq F(0) + \nicefrac{1}{2} + L\|\vx\|^2
\end{align*}
where the second inequality uses $\|\nabla F(0)\| \leq \sqrt{L}$ and the Cauchy Schwarz inequality, and the last inequality uses the identity $ab \leq a^2 + \nicefrac{b^2}{4}$. It follows that,
\begin{align*}
    \bE_{\vx \sim \mu_0}[F(\vx)] \leq F(0) + \nicefrac{1}{2} + L R^2
\end{align*}
By a slight abuse of notation, let $\mu_0$ denote the density of $\mathsf{Uniform}(\mathcal{B}(R))$. Clearly. $\mu_0(\vx) = \tfrac{1}{\vol(\cB(R))} \mathbb{I}_{\vx \in \cB(R)}$. It follows that,
\begin{align*}
    \int_{\bR^d} \mu_0(\vx) \ln(\mu_0(\vx)) \mathsf{d} \vx = \int_{\cB(R)} \tfrac{1}{\vol(\cB(R))} \log(\nicefrac{1}{\vol(\cB(R))}) \mathsf{d} \vx  =  - \log(\vol(\cB(R)))
\end{align*}
Now, $\vol(\cB(R)) = \tfrac{\pi^{\nicefrac{d}{2}}}{\Gamma(\tfrac{d}{2}+1)} R^{d}$. Furthermore, by Stirling's Approximation, $(\nicefrac{x}{e})^{x-1} \leq \Gamma(x) \leq (\nicefrac{x}{2})^{x-1}$. Hence,
\begin{align*}
    \frac{d}{2} \log\left(\frac{2 \pi R^2}{\nicefrac{d}{2} + 1}\right) \leq \log(\vol(\cB(R))) \leq \frac{d}{2} \log\left(\frac{e \pi R^2}{\nicefrac{d}{2} + 1}\right)
\end{align*}
Without loss of generality, assume $\pistar(\vx) = e^{-F(\vx)}$ (this can be easily ensured by appropriately adjusting $F(0)$ upto constant factors). It follows that,
\begin{align*}
    \KL{\mu_0}{\pistar} &= \int_{\bR^d} \mu_0(\vx) \log(\tfrac{\mu_0(\vx)}{\pistar(\vx)}) \mathsf{d} \vx = \int_{\bR^d} \mu_0(\vx) \ln(\mu_0(\vx)) \mathsf{d} \vx + \bE_{\vx \sim \mu_0}[F(\vx)] \\
    &\leq - \log(\vol(\cB(R))) + F(0) + \nicefrac{1}{2} + L R^2 \\
    &\leq \frac{d}{2} \log\left(\frac{\nicefrac{d}{2} + 1}{2 \pi R^2}\right) + F(0) + \nicefrac{1}{2} + LR^2 
\end{align*}
Setting $R = \sqrt{\nicefrac{d}{L}}$, we conclude that,
\begin{align*}
    \KL{\mu_0}{\pistar} \leq \frac{d}{2} \log(\nicefrac{L}{2 \pi}) + d + F(0) + \nicefrac{1}{2} \leq O(d)
\end{align*}
\end{proof}
We now show that the growth condition on $F$, i.e. Assumption \ref{as:growth} is  more general than specific concentration assumptions on $\pistar$ (e.g. subgaussianity, subexponentiality etc.). To this end, we define the notion of $\alpha$-tail decay as follows:
\begin{definition}[$\alpha$-Tail Decay]
\label{def:orlicz-cond}
A probability distribution $\nu$ on $\bR^d$ is said to satisfy $\alpha$-tail decay for some $\alpha > 0$ if there exists some $C > 0$ such that $\bE_{\vx \sim \nu} \left[\exp\left(\|\tfrac{\vx}{C}\|^\alpha \right)\right] < \infty$
\end{definition}
The $\alpha$-tail decay condition essentially implies that the tails of $\pistar$ decay as $\propto e^{-\|\vx\|^\alpha}$. In particular, \citet[Proposition 2.5.2 and Proposition 2.7.1]{vershynin2018high} shows that $\pistar$ satisfying the tail decay condition with $\alpha=2$ is equivalent to $\pistar$ being subgaussian, whereas tail deay with $\alpha=1$ is equivalent to $\pistar$ being subexponential.

In the following lemma, we establish that, under smoothness of $F$, the $\alpha$-tail decay condition is equivalent to the growth condition on $F$ with the same exponent $\alpha$. Consequently, Assumption \ref{as:growth} is much weaker than the standard isoperimetric and information transport assumptions generally used in the literature. 
\begin{lemma}[Growth Condition and Tail Decay]
\label{lem:growth-and-tail}
Let Assumption \ref{as:growth} be satisfied for some $\alpha > 0$. Then, $\pistar$ satisfies the $\alpha$-tail decay condition. Conversely, let Assumption \ref{as:smooth} be satisfied and suppose $\pistar$ satisfies the $\alpha$-tail decay condition. Then, $F$ satisfies Assumption \ref{as:growth} with the same exponent $\alpha$. 
\end{lemma}
\begin{proof}
\textbf{Growth Condition Implies Tail Decay} 
Since Assumption \ref{as:growth} is satisfied, $F(\vx) \geq d_1 \| \vx \|^{\alpha} - d_2$ for some $d_1, d_2, \alpha > 0$. Let $C = (\nicefrac{2}{d_1})^{\nicefrac{1}{\alpha}}$. It follows that,
\begin{align*}
    \bE_{\vx \sim \pistar}\left[e^{\| \nicefrac{\vx}{C} \|^\alpha}\right] &= \int_{\bR^d} e^{\tfrac{d_1}{2} \|\vx\|^\alpha} \pistar(\vx) \dd \vx \\
    &\leq \int_{\bR^d} e^{\tfrac{d_1}{2} \|\vx\|^\alpha - d_1 \| \vx \|^{\alpha} + d_2}  \dd \vx \\
    &= e^{d_2} \int_{\bR^d} e^{-\tfrac{d_1}{2} \|\vx\|^{\alpha}} \dd \vx < \infty
\end{align*}
From Definition \ref{def:orlicz-cond}, we conclude that $\pistar$ satisfies $\alpha$-tail decay. 

\textbf{Smoothness and Tail Decay Imply the Growth Condition} Since $F$ is smooth, it suffices to consider $\alpha \in (0, 2]$. By Assumption \ref{as:smooth}, the following inequalities hold,
\begin{align}
\label{eq:growth-lem-smooth-eq1}
F(\vy) - F(\vx) \leq \| \nabla F(\vx) \| \| \vy - \vx\| + \frac{L}{2} \| \vy - \vx \|^2 \leq (L \| \vx \| + \sqrt{L}) \| \vy - \vx \| + \frac{L}{2} \| \vy - \vx \|^2
\end{align}
Ww now prove this result by contradiction. Since $\pistar$ satisfies $\alpha$-tail decay, there exists a constant $C > 0$ such that $\bE_{\vx \sim \pistar}[e^{\|\nicefrac{\vx}{C}\|^\alpha}] < \infty$. Now, suppose $F$ does not satisfy the growth condition with exponent $\alpha$, i.e., \emph{assume there does not exist} any $d_1, d_2 > 0$ such that $F(\vx) \geq d_1 \| \vx \|^{\alpha} - d_2 \ \forall \ \vx \in \bR^d$. This implies that, $\liminf_{\| \vx \| \to \infty} \frac{F(\vx)}{\|\vx\|^{\alpha}} = 0$. Thus, without loss of generality, we can assume there exists a diverging sequence $a_n \in \bR$ and a diverging sequence $\vx_n \in \bR^d$ that satisfy the following for every $n \in \mathbb{N}$ :
\begin{align}
\label{eq:growth-lem-sequence-condition}
    \frac{F(\vx_n)}{\| \vx_n \|^{\alpha}} \leq \frac{1}{a_n}, \quad \| \vx_n\| \geq 2n, \quad \| \vx_{n+1} - \vx_{n} \| \geq 1
\end{align}
where, without loss of generality, we assume $a_n, \| \vx_n \|$ > 0. Now, let $r_n = \tfrac{1}{\|\vx_n\|^2}$ and $B_n \subseteq \bR^d$ denote the ball of radius $r_n$ centered at $\vx_n$. Since $r_n \leq \nicefrac{1}{4n^2}$ and $\| \vx_{n+1} - \vx_n \| \geq 1$, $B_n$ is a family of disjoint subsets of $\bR^d$. We shall now prove that there exists some diverging sequence $b_n \in \bR$ such that $\tfrac{F(\vy)}{\|\vy\|^\alpha} \leq \tfrac{1}{b_n}$ for every $\vy \in B_n$. 

Consider any arbitrary $n \in \mathbb{N}$ and let $\vy \in B_n$. Applying \eqref{eq:growth-lem-smooth-eq1} to $\vy$ and $\vx_n$, we obtain,
\begin{align}
\label{eq:growth-lem-eq2}
    \frac{F(\vy)}{\| \vx_n \|^\alpha} &\leq \frac{F(\vx_n)}{\|\vx_n\|^\alpha} + \frac{L \|\vx_n\|r_n}{\|\vx_n\|^{\alpha}} + \frac{r_n\sqrt{L}}{\|\vx_n\|^\alpha} + \frac{L r^2_n}{2 \| \vx_n \|^\alpha} \nonumber \\
    &\leq \frac{1}{a_n} + \frac{L}{\| \vx \|^{\alpha + 1}} + \frac{\sqrt{L}}{\| \vx_n \|^{\alpha + 2}} + \frac{L}{2 \| \vx_n \|^{\alpha + 4}} 
\end{align}
where we use \eqref{eq:growth-lem-sequence-condition} and $r_n = \nicefrac{1}{\|\vx_n\|^2}$. Moreover, we note that 
\begin{align}
\label{eq:growth-lem-ynorm}
    \| \vy \| \geq \| \vx_n \| - \| \vy - \vx_n \| \geq \| \vx_n \| - r_n = \|\vx_n\| - \frac{1}{\|\vx_n\|^2} \geq \frac{\|\vx_n\|}{2}
\end{align}
where we use the fact that $\| \vx_n \| \geq 2n > 2^{\nicefrac{1}{3}}$. It follows that,
\begin{align}
\label{eq:growth-lem-eq3}
    \frac{F(\vy)}{\|\vy\|^{\alpha}} &\leq \frac{2^{\alpha} F(\vy)}{\| \vx_n \|^{\alpha}} \nonumber \\
    &\leq \frac{4}{a_n} + \frac{4L}{\| \vx \|^{\alpha + 1}} + \frac{4\sqrt{L}}{\| \vx_n \|^{\alpha + 2}} + \frac{2L}{\| \vx_n \|^{\alpha + 4}} 
\end{align}
where we use \eqref{eq:growth-lem-eq2} and the fact that $\alpha \in (0, 2]$. We now define the sequence $b_n \in \bR$ as follows:
\begin{align*}
    b_n &= \left(\frac{4}{a_n} + \frac{4L}{\| \vx \|^{\alpha + 1}} + \frac{4\sqrt{L}}{\| \vx_n \|^{\alpha + 2}} + \frac{2L}{\| \vx_n \|^{\alpha + 4}} \right)^{-1}
\end{align*}
Since $\alpha > 0$, and $a_n, \| \vx_n \| \to \infty$, it is clear that $b_n$ is a diverging sequence. Furthermore, from \eqref{eq:growth-lem-eq3}, we conclude that $\tfrac{F(\vy)}{\|\vy\|^{\alpha}} \leq \tfrac{1}{b_n} \ \forall \ \vy \in B_n$. Equipped with this construction, we note that
\begin{align*}
    \bE_{\vx \sim \pistar}\left[\exp\left(\frac{\|\vx\|^{\alpha}}{C^{\alpha}}\right)\right] &= \int_{\bR^d} \exp\left(\frac{\|\vy\|^\alpha}{C^\alpha}\right) \exp(-F(\vy)) \dd \vy \\
    &\geq \sum_{n=1}^{\infty} \int_{B_n} \exp\left(\frac{\|\vy\|^\alpha}{C^\alpha}\right) \exp(-F(\vy)) \dd \vy \\
    &\geq \sum_{n=1}^{\infty} \int_{B_n} \exp\left(\frac{\|\vy\|^\alpha}{C^\alpha} -\frac{\|\vy\|^\alpha}{b_n}\right) \dd \vy
\end{align*}
where the second inequality use the fact that $B_n$ is a disjoint family of subsets of $\bR^d$ and the third inequality uses the fact that $\tfrac{F(\vy)}{\|\vy\|^{\alpha}} \leq \tfrac{1}{b_n} \ \forall \ \vy \in B_n$. Since $b_n$ is a diverging sequence, there exists some $N_0 \in \mathbb{N}$ such that $b_n \geq 2 C^{\alpha} \ \forall \ n \geq N_0$. It follows that,
\begin{align*}
    \bE_{\vx \sim \pistar}\left[\exp\left(\frac{\|\vx\|^{\alpha}}{C^{\alpha}}\right)\right] &\geq \sum_{n=1}^{\infty} \int_{B_n} \exp\left(\frac{\|\vy\|^\alpha}{C^\alpha} -\frac{\|\vy\|^\alpha}{b_n}\right) \dd \vy \\
    &\geq \sum_{n=N_0}^{\infty} \int_{B_n} \exp\left(\frac{\|\vy\|^\alpha}{2C^\alpha}\right) \dd \vy \\
    &= \sum_{n=N_0}^{\infty} \vol(B_n) \bE_{\vy \sim \Uniform(B_n)}\left[\exp\left(\frac{\|\vy\|^\alpha}{2C^\alpha}\right)\right]
\end{align*}
Consider the function $g : [0, \infty) \to [0, \infty)$ defined as $g(t) = e^{t^{\alpha}}$. We note that for $\alpha \geq 1$, $g$ is a convex function for every $t \geq 0$, and for $\alpha \in (0,1)$, $g$ is convex for every $t \geq (\nicefrac{1}{\alpha} - 1)^{\nicefrac{1}{\alpha}}$. From \eqref{eq:growth-lem-ynorm}, we note that $\|\vy\| \geq \nicefrac{\|\vx\|}{2} \geq n$ for every $\vy \in B_n$. Hence, there exists an $N_1 \in \mathbb{N}$ such that $e^{t^{\alpha}}$ is a convex function for all $t \geq \nicefrac{\|\vy\|}{2}, \ \forall \ \vy \in B_n, \ n \geq N_1$. Let $N = \max\{N_0, N_1\} + 1$. Then,
\begin{align*}
    \bE_{\vx \sim \pistar}\left[\exp\left(\frac{\|\vx\|^{\alpha}}{C^{\alpha}}\right)\right] &\geq \sum_{n=N_0}^{\infty} \vol(B_n) \bE_{ \vy \sim \Uniform(B_n)}\left[\exp\left(\frac{\|\vy\|^\alpha}{2C^\alpha}\right)\right] \\
    &\geq \sum_{n=N}^{\infty} \vol(B_n) \exp\left(\tfrac{1}{2C^{\alpha}} \bE_{\vy \sim \Uniform(B_n)}[\|\vy\|]^\alpha\right) \\
    &\geq \sum_{n=N}^{\infty} \vol(B_n) \exp\left(\tfrac{1}{2C^{\alpha}} \|\bE_{\vy \sim \Uniform(B_n)}[\vy]\|^\alpha\right) \\
    &\geq \sum_{n=N}^{\infty} \vol(B_n) \exp\left(\tfrac{1}{2C^{\alpha}} \|\vx_n\|^\alpha\right) \\
    &= \sum_{n=N}^{\infty} C_d (r_n)^{d}\exp\left(\tfrac{1}{2C^{\alpha}} \|\vx_n\|^\alpha\right) \\
    &= \sum_{n=N}^{\infty} \frac{C_d \exp\left(\tfrac{1}{2C^{\alpha}} \|\vx_n\|^\alpha\right)}{\|\vx_n \|^{2d}}
\end{align*}
where $C_d = \tfrac{\pi^{\nicefrac{d}{2}}}{\Gamma(\nicefrac{d}{2} + 1)}$. Let $k$ be any positive integer such that $\alpha k \geq 2d+1$. It follows that,
\begin{align*}
    \frac{\exp\left(\tfrac{1}{2C^{\alpha}} \|\vx_n\|^\alpha\right)}{\|\vx_n \|^{2d}} \geq \frac{C_d}{2^k k! C^{\alpha k}} \| \vx_{n} \|^{\alpha k - 2d} \geq \frac{C_d n}{2^{k-1} k! C^{\alpha k}} 
\end{align*}
Thus, we infer that,
\begin{align*}
    \bE_{\vx \sim \pistar}\left[\exp\left(\frac{\|\vx\|^{\alpha}}{C^{\alpha}}\right)\right] \geq \frac{C_d}{2^{k-1} k! C^{\alpha k}} \sum_{n=N_0}^{\infty} n = \infty
\end{align*}
which is a contradiction. Thus, there exists some $d_1, d_2 > 0$ such that $F(\vx) \geq d_1 \| \vx \|^{\alpha} - d_2$, i.e., $F$ satisfies the growth condition with exponent $\alpha$.
\end{proof}
The following lemma establishes boundedness and contractivity properties of the function $h(\vx, \vy) = k(\vx, \vy) \nabla F(\vy) - \nabla_2 k(\cdot, \vy)$, that are vital for proving almost-sure bounds such as Lemma \ref{lem:gt-bound-as}.
\begin{lemma}[Properties of $h$]
\label{lem:h-bounds}
Let Assumptions \ref{as:smooth}, \ref{as:norm} and \ref{as:kern_decay} be satisfied. Then, the following holds,
\begin{align*}
    \| h(., \vy) \|_{\cH} &\leq BL \| \vy \| + B\| \nabla F(0) \| + B \\
    \|h(\vx, \vy)\| &\leq \frac{A_1 L}{2} + A_2 + k(\vx, \vy) \| \nabla F(\vx) \| \\
    -\dotp{\nabla F(\vx)}{h(\vx, \vy)} &\leq -\tfrac{1}{2} k(\vx, \vy) \| \nabla F(\vx) \|^2 + L^2 A_1 + A_3 
\end{align*}
\end{lemma}
\begin{proof}
Recalling the definition of $h$ from Section \ref{sec:notation}, we observe that,
\begin{align*}
    h(\cdot, \vy) = k(\cdot, \vy) \nabla F(\vy) - \nabla_2 k(\cdot, \vy)
\end{align*}
Thus, by triangle inequality of $\|\cdot\|_{\cH}$, Assumptions \ref{as:smooth} and \ref{as:norm}, we obtain
\begin{align*}
    \|h(\cdot, \vy)\|_{\cH} &\leq \| \nabla F(\vy) \| \| k(\cdot, \vy) \|_{\cH_0} + \|\nabla_2 k(\cdot, \vy) \|_{\cH} \\
    &\leq B L \|\vy \| + B \| \nabla F(0) \| + B
\end{align*}
To prove the remaining inequalities, we first note that, 
\begin{align}
\label{eq:h-decomp}
h(\vx, \vy) &= k(\vx, \vy) \nabla F(\vy) - \nabla_2 k(\vx, \vy) \nonumber \\
&= k(\vx, \vy) \nabla F(\vx) + k(\vx, \vy) \left[ \nabla F(\vy) - \nabla F(\vx) \right] - \nabla_2 k(\vx, \vy)
\end{align}
Using Assumptions \ref{as:smooth} and \ref{as:kern_decay}, we note that,
\begin{align*}
    \| h(\vx, \vy) \| &\leq k(\vx, \vy) \| \nabla F(\vx) \| + \frac{L A_1 \|\vx - \vy\|}{1 + \| \vx - \vy \|^2} + A_2  \\
    &\leq \frac{A_1 L}{2} + A_2 + k(\vx, \vy) \| \nabla F(\vx) \|
\end{align*}
where the second inequality uses the fact $\tfrac{t}{1+t^2} \leq \nicefrac{1}{2}$

To prove the last inequality, we infer the following from \eqref{eq:h-decomp}
\begin{align*}
    - \dotp{\nabla F(\vx)}{h(\vx, \vy)} &\leq - k(\vx, \vy) \| \nabla F(\vx) \|^2 + k(\vx, \vy) \| \nabla F(\vx) - \nabla F(\vy) \| \| \nabla F(\vx) \|  \\
    &+ \| \nabla_2 k(\vx, \vy)\| \|\nabla F(\vx)\|  \\
    &\leq - k(\vx, \vy) \| \nabla F(\vx)\|^2 + L \sqrt{k(\vx, \vy)} \sqrt{\frac{A_1 \|\vx - \vy\|^2}{1 + \|\vx - \vy\|^2}} \| \nabla F(\vx) \|  \\
    &+ \sqrt{A_3 k(\vx, \vy)} \| \nabla F(\vx) \|  \\
    &\leq -\frac{1}{2} k(\vx, \vy) \|\nabla F(\vx)\|^2 + L^2 A_1 + A_3 
\end{align*}
where the second inequality uses Assumptions \ref{as:smooth} and \ref{as:kern_decay}, and the last inequality uses the identity $ab \leq a^2 + \nicefrac{b^2}{4}$
\end{proof}
To analyze the dynamics of \ouralg~in the Wasserstein space, we use the following lemma presented in \citet{salim2022convergence} 
\begin{lemma}[\citet{salim2022convergence}, Proposition 3.1]
\label{lem:salim-pop-kl-descent}
Let Assumptions \ref{as:smooth} and \ref{as:norm} be satisfied. Consider any $\nu_0 \in \cP_2(\bR^d)$ with $\KL{\nu_0}{\pistar} < \infty$, $f \in \cH$ and let $\nu_1 = (I - \eta f)_{\#}\nu_0$ with $\eta \| f \|_{\cH} \leq \tfrac{\beta - 1}{\beta B}$ for some $\beta > 1$. Then, the following holds,
\begin{align*}
    \KL{\mu_1}{\pistar} \leq \KL{\mu_0}{\pistar} - \eta \dotp{h_{\mu_0}}{f} + \frac{\eta^2 (\beta^2 + L) B}{2} \|f\|^2_{\cH}
\end{align*}
\end{lemma}
\subsection{Gelfand-Pettis Integrals for Reproducing Kernel Hilbert Spaces}
\label{app-sec:gelfand-pettis}
The Gelfand-Pettis integral is a generalization of the Lebesgue integral to functions that take values in an arbitrary topological vector space. In this section, we describe the Gelfand-Pettis integral for an arbitrary Hilbert space $(V, \dotp{\cdot}{\cdot}_{V})$ and refer the readers to \citet{talagrand1984pettis} for a more general treatment. 

Let $(X, \Sigma, \lambda)$ be a measure space and $(V, \dotp{\cdot}{\cdot}_{V})$ be a Hilbert Space. A function $g : X \to V$ is said to be Gelfand-Pettis integrable  if there exists a vector $w_g \in V$ such that $\dotp{u}{w_g}_V = \int_{X} \dotp{u}{g(x)}_V \dd \lambda(x) \ \forall \ u \in V$. The vector $w_g$ is called the Gelfand-Pettis integral of $g$ 

We now establish the following lemma for Gelfand-Pettis integrals with respect to the RKHS $\cH$, which is a key component of our analysis of \ouralg.
\begin{lemma}
\label{lem:pettis-aux-lem}
Let $\mu$ be a probability measure on $\bR^d$. Let $G : \bR^d \times \bR^d$ be a function such that for every $\vy \in \bR^d$, $G(., \vy) \in \cH$ with $\| G(., \vy) \|_\cH \leq C$ holding $\mu$-almost surely. Let $G_{\mu}(\vx) = \bE_{\vy \sim \mu}[G(\vx, \vy)]$. Then, the map $\psi : \bR^d \to \cH$ defined as $\psi(\vy) = G(\cdot, \vy)$ is Gelfand-Pettis integrable and $G_\mu$ is the Gelfand-Pettis integral of $\psi$ with respect to $\mu$, i.e. $G_\mu \in \cH$ and for any $f \in \cH$, $\bE_{\vy \sim \mu}[\dotp{f}{G(.,\vy)}_{\cH}] = \dotp{f}{G_\mu}_\cH$    
\end{lemma}
\begin{proof}
Let $\Phi : \cH \to \bR$ denote the map $\Phi(f) = \bE_{\vy \sim \mu}[\dotp{f}{G(., \vy)}_\cH] \ \forall \ f \in \cH$. By linearity of expectations and inner products, we note that $\Phi$ is a linear functional on $\cH$. Furthermore, since $\| G(., \vy) \|_\cH \leq C$ holds $\mu$-almost surely, we note that for any $f \in \cH$, $|\Phi(f)| \leq \bE_{\vy \sim \mu}[|\dotp{f}{G(., \vy)}_\cH|] \leq C\| f \|_{\cH}$ by Jensen's inequality and Cauchy Schwarz inequality for $\cH$. We conclude that $\Phi$ is a bounded linear functional of $\cH$. Thus, by Reisz Representation Theorem \citep{conway2019functional}, there exists $g \in \cH$ such that for any $f \in \cH$, the following holds
\begin{align*}
    \bE_{\vy \in \mu}[\dotp{f}{G(., \vy)}_\cH] = \dotp{f}{g}_\cH
\end{align*}
Hence, we conclude that the map $\psi$ is Gelfand-Pettis integrable. We now use the reproducing property of $\cH$ to show that $g = G_{\mu}$, i.e., $G_\mu$ is the Gelfand-Pettis integral of $\psi$. To this end, let $\vx \in \bR^d$ be arbitrary. Setting $f = k(\vx, .)$ and using the fact that $g \in \cH$, $G(., \vy) \in \cH$ for any $\vy \in \bR^d$,
\begin{align*}
    g(\vx) = \bE_{\vy \in \mu}[G(\vx, \vy)] = G_{\mu}(\vx) 
\end{align*}
Hence, $g=G_\mu$, i.e.,  $\bE_{\vy \sim \mu}[\dotp{f}{G(.,\vy)}_{\cH}] = \dotp{f}{G_\mu}_\cH$    
\end{proof}
\section{Analysis of \ouralg}
\label{app-sec:vpsvgd-analysis}
In this section, we present our analysis of \ouralg. Throughout this section, we define the random function $g_t : \bR^d \times \bR^d$ as $g_t(\vx) = \tfrac{1}{K} \sum_{l=0}^{K-1} h(\vx, \vx^{(Kt+l)}_t)$ where $t \in \mathbb{N} \cup \{0\}$, $K$ is the batch-size of \ouralg, and $h : \bR^d \times \bR^d$ is as defined in Section \ref{sec:notation}, i.e., $h(\vx, \vy) = k(\vx, \vy) \nabla F(\vy) - \nabla_2 k(\vx, \vy)$. 

After proving the key lemmas required for our analysis of \ouralg, we present the proof of Theorem \ref{thm:vpsvgd-bound} in Appendix \ref{prf:vpsvgd-bound-thm}. We also present a high-probability version of Theorem \ref{thm:vpsvgd-bound} in Appendix \ref{prf:vpsvgd-hp-bound-thm}
\subsection{Population Level Dynamics : Proof of Lemma \ref{lem:evi-lemma}}
\label{prf:vpsvgd-evi-lemma-proof}
\begin{proof}
We now derive the population-limit dynamics of \ouralg~for arbitrary batch-size $K$, and subsequently prove the descent lemma (i.e. Lemma \ref{lem:evi-lemma}) for \ouralg. The arguments of this section are a straightforward generalization of that used in Section \ref{sec:alg}. 

To this end, we recall from Section \ref{sec:alg} that the countably infinite number of particles $\vx^{(l)}_0$, $l \in \mathbb{N} \cup \{0\}$ are i.i.d samples from the measure $\mu_0$, which has a density w.r.t the Lebesgue measure. Thus, by the strong law of large numbers (\citet[Theorem 11.4.1]{dudley2018real}), the empirical measure of $(\vx^{(l)}_0)_{l \geq 0}$ is almost surely equal to $\mu_0$. Furthermore, we recall the filtration $\cF_t$ defined in Section \ref{sec:alg} as $\cF_t = \sigma(\vx^{(l)}_0 \ | \ l \leq Kt - 1)$, $t \in \mathbb{N}$ with $\cF_0$ being the trivial $\sigma$ algebra. We now consider the following dynamics in $\bR^d$:
\begin{equation} 
\label{eq:batch-diagonal-dynamics}
    \vx^{(s)}_{t+1} = \vx^{(s)}_t - \frac{\gamma}{K} \sum_{l=0}^{K-1} h(\vx^{(s)}_t, \vx^{(tK+l)}_t), \quad s \in \mathbb{N} \cup \{0\}
\end{equation}
We note that the above updates are the same as that of \ouralg~for $s \in \{0, \dots, KT+n-1\}$. Now, for each time-step $t$, we focus on the lower triangular evolution, i.e., the time evolution of the particles $(\vx^{(l)}_t)_{l \geq Kt}$. From \eqref{eq:batch-diagonal-dynamics}, we infer that for any $t \in \mathbb{N}$ and $s \geq Kt$, $\vx^{(s)}_t$ depends only on $(\vx^{(l)}_0)_{l \leq Kt-1}$ and $\vx^{(s)}_0$. Hence, there exists a measurable function $H_t$ for every $t \in \mathbb{N}$ such that the following holds almost surely:
\begin{equation}
\label{eq:vpsvgd-rand-func-rep}
    \vx^{(s)}_t = H_t(\vx^{(0)}_0, \dots, \vx^{(Kt-1)}_0, \vx^{(s)}_0); \quad \forall \ s \geq Kt
\end{equation}
Since $\vx^{(0)}_0, \dots, \vx^{(Kt-1)}_0, \vx^{(s)}_0 \ \iidsim \ \mu_0$, we conclude from \eqref{eq:vpsvgd-rand-func-rep} that $(\vx^{(s)}_t)_{s \geq Kt}$ are i.i.d when conditioned on $\vx^{(0)}_0, \dots, \vx^{(Kt-1)}_0$. To this end, we define the random measure $\mu_t | \cF_t$ as the law of $\vx^{(Kt)}_t$ conditioned on $\cF_t$, i.e. $\mu_t | \cF_t$ is a probability kernel $\mu_t(\cdot; \vx^{(0)}_0, \dots, \vx^{(Kt-1)}_0)$ with $\mu_0 | \cF_0 \coloneqq \mu_0$. By the strong law of large numbers, $\mu_t | \cF_t$ is equal to the empirical measure of $(\vx^{(l)}_t)_{l \geq Kt}$ conditioned on $\cF_t$. Furthermore, we infer from \eqref{eq:batch-diagonal-dynamics} that the particles satisfy the following:
\begin{align*}
    \vx^{(s)}_{t+1} = (I - \gamma g_t)(\vx^{(s)}_t), \quad s \geq K(t+1)
\end{align*}
Recall that $\vx_{t+1}^{(s)} | \vx^{(0)}_0, \dots, \vx^{(K(t+1)-1)}_0 \sim \mu_{t+1} | \cF_{t+1}$ for any $s \geq K(t+1)$. Furthermore, from  Equation~\eqref{eq:vpsvgd-rand-func-rep}, we note that for $s \geq K(t+1)$, $\vx^{(s)}_t$ depends only on $\vx^{(0)}_0, \dots, \vx^{(Kt-1)}_0$ and $\vx^{(s)}_0$, which implies that $\mathrm{Law}(\vx^{(s)}_t | \vx^{(0)}_0, \dots, \vx^{(K(t+1)-1)}_0) = \mathrm{Law}(\vx^{(s)}_t | \vx^{(0)}_0, \dots, \vx^{(Kt-1)}_0) = \mu_t | \cF_t$. Finally, we note that $g_t$ is an $\cF_{t+1}$-measurable random function. With these insights, we conclude that the population-level dynamics of the lower triangular evolution in $\cP_2(\bR^d)$ is almost surely described by the following update:
\begin{equation}
\label{eq:batch-diagonal-dynamics-pop}
    \mu_{t+1} | \cF_{t+1} = (I - \gamma g_t)_{\#} \mu_t | \cF_t
\end{equation}
Setting $\gamma \| g_t \|_{\cH} \leq \tfrac{\beta - 1}{\beta B}$ for some arbitrary $\beta > 1$ and applying Lemma \ref{lem:salim-pop-kl-descent} to the population-level update \eqref{eq:batch-diagonal-dynamics-pop}, we conclude that the following holds almost surely:
\begin{align*}
    \KL{\mu_{t+1} | \cF_{t+1}}{\pistar} \leq \KL{\mu_{t} | \cF_{t}}{\pistar} - \gamma \dotp{h_{\mu_t | \cF_t}}{g_t}_{\cH} + \frac{\gamma^2 (\beta^2 + L)B}{2} \|g_t\|^{2}_{\cH}
\end{align*}
\end{proof}
\subsection{Iterate Bounds : Proof of Lemma \ref{lem:gt-bound-as}}
\label{prf:vpsvgd-itr-bound-proof}
To establish almost sure bounds on $\|g_t\|_{\cH}$, we prove the following result which is stronger than Lemma \ref{lem:gt-bound-as}. 
\begin{lemma}[Almost-Sure Iterate Bounds for \ouralg]
\label{lem:vpsvgd-itr-bounds}
Let Assumptions \ref{as:smooth}, \ref{as:growth}, \ref{as:norm}, \ref{as:kern_decay} and \ref{as:init} be satisfied. Then, the following holds almost surely for any $s \in \mathbb{N} \cup \{0\}$ and $t \in (T+1)$ whenever $\gamma \leq \nicefrac{1}{2 A_1 L}$
\begin{align*}
    \| \vx^{(s)}_t \| &\leq \zeta_0 + \zeta_1 (\gamma T)^{\nicefrac{1}{\alpha}} + \zeta_2 (\gamma^2  T)^{\nicefrac{1}{\alpha}} + \zeta_3 R^{\nicefrac{2}{\alpha}} \\
    \| h(\cdot, \vx^{(s)}_t) \|_{\cH} &\leq \zeta_0 + \zeta_1 (\gamma T)^{\nicefrac{1}{\alpha}} + \zeta_2 (\gamma^2  T)^{\nicefrac{1}{\alpha}} + \zeta_3 R^{\nicefrac{2}{\alpha}} \\
    \| g_t \|_{\cH} &\leq \zeta_0 + \zeta_1 (\gamma T)^{\nicefrac{1}{\alpha}} + \zeta_2 (\gamma^2  T)^{\nicefrac{1}{\alpha}} + \zeta_3 R^{\nicefrac{2}{\alpha}}
\end{align*}
where $\zeta_0, \dots, \zeta_3$ are problem-dependent constants that depend polynomially on $A_1, A_2, A_3, B, d_1, d_2, L$ for any fixed $\alpha$. 
\end{lemma}
\begin{proof}
Let $c^{(s)}_t = \tfrac{1}{K} \sum_{l=0}^{K-1} k(\vx^{(s)}_t, \vx^{(Kt+l)}_t)$. Note that by Assumption \ref{as:kern_decay}, $c^{(s)}_t \geq 0$ Since $\vx^{(s)}_{t+1} = \vx^{(s)}_t - \gamma g_t(\vx^{(s)}_t)$, it follows from the smoothness of $F$ that,
\begin{equation}
\label{eq:vpsvgd-path-descent}
    F(\vx^{(s)}_{t+1}) - F(\vx^{(s)}) \leq - \gamma \dotp{\nabla F(\vx^{(s)}_t)}{g_t(\vx^{(s)}_t)} + \frac{\gamma^2 L}{2} \| g_t(\vx^{(s)}_t) \|^2 
\end{equation}
By Lemma \ref{lem:h-bounds}, we note that,
\begin{align}
\label{eq:vpsvgd-descent-term1}
    - \gamma \dotp{\nabla F(\vx^{(s)}_t)}{g_t(\vx^{(s)}_t)} &= -\frac{\gamma}{K} 
\sum_{l=0}^{K-1}\dotp{\nabla F(\vx^{(s)}_t)}{h(\vx^{(s)}_t, \vx^{(tK+l)}_t)} \nonumber \\
&\leq \frac{\gamma}{K} \sum_{l=0}^{L-1} \left[-\tfrac{1}{2}k(\vx^{(s)}_t, \vx^{(tK+l)}_t) \| \nabla F(\vx^{(s)}_t)\|^2 + L^2 A_1 + A_3 \right] \nonumber \\
&\leq -\frac{\gamma c^{(s)}_t}{2} \| \nabla F(\vx^{(s)}_t)\|^2 + \gamma L^2 A_1 + \gamma A_3 
\end{align}
Moreover, by Jensen's Inequality and Lemma \ref{lem:h-bounds}
\begin{align}
\label{eq:vpsvgd-descent-term2}
\| g_t(\vx^{(s)}_t) \|^2 &\leq \frac{1}{K} \sum_{l=0}^{K-1} \| h(\vx^{(s)}_t, \vx^{(Kt+l)}_t) \|^2 \nonumber \\
&\leq \frac{1}{K} \sum_{l=0}^{K-1} 2(\nicefrac{A_1 L}{2} + A_2)^2 + 2 k(\vx^{(s)}_t, \vx^{(tK+l)}_t)^2 \| F(\vx^{(s)}_t) \|^2 \nonumber \\
&\leq \frac{1}{K} \sum_{l=0}^{K-1} 2(\nicefrac{A_1 L}{2} + A_2)^2 + 2 A_1 k(\vx^{(s)}_t, \vx^{(tK+l)}_t) \| F(\vx^{(s)}_t) \|^2 \nonumber \\
&\leq 2(\nicefrac{A_1 L}{2} + A_2)^2 + 2 A_1 c^{(s)}_t\| F(\vx^{(s)}_t) \|^2
\end{align}
Substituting \eqref{eq:vpsvgd-descent-term1} and \eqref{eq:vpsvgd-descent-term2} into \eqref{eq:vpsvgd-path-descent}, we obtain,
\begin{align*}
    F(\vx^{(s)}_{t+1}) - F(\vx^{(s)}_t) &\leq -\frac{\gamma c^{(s)}_t}{2} \| \nabla F(\vx^{(s)}_t)\|^2 + \gamma L^2 A_1 + \gamma A_3 \\
    &+ \gamma^2 L(\nicefrac{A_1 L}{2} + A_2)^2 + \gamma^2 L A_1 c^{(s)}_t\| F(\vx^{(s)}_t) \|^2 \\
    &\leq -\frac{\gamma c^{(s)}_t}{2} (1 - 2 A_1 L \gamma) \| \nabla F(\vx^{(s)}_t) \|^2 + \gamma A_3 + \gamma L^2 A_1 + \gamma^2 L (\nicefrac{A_1 L}{2} + A_2)^2 \\
    &\leq \gamma A_3 + \gamma L^2 A_1 + \gamma^2 L (\nicefrac{A_1 L}{2} + A_2)^2
\end{align*}
where the last inequality uses the fact that $c^{(s)}_t \geq 0$ and $\gamma \leq \nicefrac{1}{2 A_1 L}$. Now, iterating through the above inequality, we obtain the following for any $t \in [T], \ s \in \mathbb{N} \cup \{ 0 \}$
\begin{equation}
\label{eq:vpsvgd-descent-unroll}    
    F(\vx^{(s)}_t) \leq F(\vx^{(s)}_0) + \gamma T L^2 A_1 + \gamma T A_3 + \gamma^2 T L (\nicefrac{A_1 L}{2} + A_2)^2
\end{equation}
Furthermore, by Assumption \ref{as:smooth}
\begin{align*}
    F(\vx^{(s)}_0) &\leq F(0) + \| \nabla F(0) \| \| \vx^{(s)}_0 \| + \frac{L}{2} \| \vx^{(s)}_0 \|^2 \\
    &\leq F(0) + \nicefrac{1}{2} + L \| \vx^{(s)}_0 \|^2
\end{align*}
Substituting the above inequality into \eqref{eq:vpsvgd-descent-unroll}, and using Assumption \ref{as:growth}, we obtain the following for any $t \in [T], \ s \in \mathbb{N} \cup \{0\}$
\begin{align*}
    d_1 \| \vx^{(s)}_t \|^\alpha - d_2 \leq F(\vx^{s}_t) &\leq F(0) + \nicefrac{1}{2} + L \| \vx^{(s)}_0 \|^2 + \gamma T L^2 A_1 + \gamma T A_3 \\
    &+ \gamma^2 T L (\nicefrac{A_1 L}{2} + A_2)^2
\end{align*}
Rearranging and applying Assumption \ref{as:init}, we obtain
\begin{align*}
    \| \vx^{(s)}_t \| &\leq d^{-\nicefrac{1}{\alpha}}_1 \left[ F(0) + \nicefrac{1}{2} + L R^2 + \gamma T L^2 A_1 + \gamma T A_3 + \gamma^2 T L(A_1L+A_2)^2  \right]^{\nicefrac{1}{\alpha}} \\
    &\leq \tilde{\zeta}_0 + \tilde{\zeta}_1 (\gamma T)^{\nicefrac{1}{\alpha}} + \tilde{\zeta}_2 (\gamma^2 T)^{\nicefrac{1}{\alpha}} + \tilde{\zeta}_3 R^{\nicefrac{2}{\alpha}}
\end{align*}
where $\tilde{\zeta}_0, \dots, \tilde{\zeta}_3$ are constants that depend polynomially on $L, A_1, A_2, A_3, R$. We note that, since $0 < \alpha \leq 2$, the above inequality also holds for $t = 0$.

Using the above inequality Lemma \ref{lem:h-bounds} and Assumption \ref{as:smooth}, we conclude that the following holds almost surely for any $t \in (T+1), s \in \mathbb{N} \cup \{ 0 \}$
\begin{align*}
    \| h(\cdot , \vx^{(s)}_t) \|_{\cH} &\leq BL \| \vx^{(s)}_t \| + B\sqrt{L} + B \\
    &\leq \tilde{\eta_0} + \tilde{\eta}_1 (\gamma T)^{\nicefrac{1}{\alpha}} + \tilde{\eta}_2 (\gamma^2 T)^{\nicefrac{1}{\alpha}} + \tilde{\eta}_3 R^{\nicefrac{2}{\alpha}}
\end{align*}
where $\tilde{\eta}_0, \dots, \tilde{\eta}_3$ are constants that depend polynomially on $L, B, A_1, A_2, A_3, R$. Using the above inequality, we conclude that the following also holds for any $t \in (T+1)$.
\begin{align*}
    \| g_t \|_{\cH} \leq \tilde{\eta_0} + \tilde{\eta}_1 (\gamma T)^{\nicefrac{1}{\alpha}} + \tilde{\eta}_2 (\gamma^2 T)^{\nicefrac{1}{\alpha}} + \tilde{\eta}_3 R^{\nicefrac{2}{\alpha}}
\end{align*}
Taking $\zeta_i = \max \{ \tilde{\zeta_i}, \tilde{\eta}_i \}$, the proof is complete. 
\end{proof}
\subsection{Controlling $g_t$ in Expectation : Proof of Lemma \ref{lem:pettis-exp}}
\label{prf:vpsvgd-pettis-emma-proof}
\label{prf:pettis-exp-lemma}
\begin{proof}
Let $\xi = \zeta_0 + \zeta_1 (\gamma T)^{\nicefrac{1}{\alpha}} + \zeta_2(\gamma^2 T)^{\nicefrac{1}{\alpha}} + \zeta_3 R^{\nicefrac{2}{\alpha}}$ where $\zeta_0, \dots, \zeta_3$ are as defined in Lemma \ref{lem:vpsvgd-itr-bounds}. Recall that $g_t = \tfrac{1}{K} \sum_{l=0}^{K-1} h(., \vx^{(Kt + l)}_t)$. Since $\gamma \leq \nicefrac{1}{2A_1 L}$, $\| h(\cdot, \vx^{(Kt+l)}_t) \|_{\cH} \leq \xi$ holds almost surely y Lemma \ref{lem:vpsvgd-itr-bounds}. 

Consider any $l \in (K)$. Conditioned on the filtration $\cF_t$, $\mathrm{Law}(\vx^{(Kt + l)}_t | \cF_t) = \mu_t | \cF_t$. Moreover, for any $\vx \in \bR^d$, $\bE_{\vx^{(Kt+l)}_t}[h(\vx, \vx^{(Kt+l)}_t) | \cF_t] = h_{\mu_t | \cF_t}(\vx)$. Thus, from Lemma \ref{lem:pettis-aux-lem}, we conclude that $h_{\mu_t | \cF_t}$ is the Gelfand-Pettis Integral of the map $\vx \to h(\vx, \vx^{(Kt+l)}_t)$ with respect to $\mu_t | \cF_t$. Hence, the following holds
\begin{align*}
    \bE_{\vx^{(Kt+l)}_t} \left[ \dotp{h(\cdot, \vx^{(Kt+l)}_t)}{f}_{\cH} \biggr| \cF_t\right] = \dotp{h_{\mu_t | \cF_t}}{f}_{\cH}
\end{align*}
In particular, setting $f = h_{\mu_t | \cF_t}$ and using linearity of expectation, we conclude,
\begin{align*}
    \bE\left[\dotp{g_t}{h_{\mu_t | \cF_t}}_{\cH} | \cF_t\right] &= \frac{1}{K} \sum_{l=0}^{K-1} \bE_{\vx^{(Kt+l)}_t}\left[\dotp{h(\cdot, \vx^{(Kt+l)}_t)}{h_{\mu_t | \cF_t}}_{\cH} \biggr| \cF_t\right] \\
    &= \| h_{\mu_t | \cF_t} \|_{\cH}^{2}
\end{align*}
To control $\bE[\| g_t \|^{2}_{\cH} | \cF_t]$, we note that,
\begin{align*}
    \|g_t\|^2_{\cH} &= \frac{1}{K^2} \sum_{l_1, l_2 = 0}^{K-1} \dotp{h(\cdot, \vx^{(Kt + l_1)}_t)}{h(\cdot, \vx^{(Kt + l_2)}_t)}_{\cH} \\
    &= \frac{1}{K^2} \sum_{l=0}^{K-1} \| h(\cdot, \vx^{(Kt+l)}_t) \|^2_\cH + \sum_{0 \leq l_1 \neq l_2 \leq K-1} \dotp{h(\cdot, \vx^{(Kt + l_1)}_t)}{h(\cdot, \vx^{(Kt + l_2)}_t)}_\cH  \\
    &\leq \frac{\xi^2}{K} + \sum_{0 \leq l_1 \neq l_2 \leq K-1} \dotp{h(\cdot, \vx^{(Kt + l_1)}_t)}{h(\cdot, \vx^{(Kt + l_2)}_t)}_\cH  
\end{align*}
where the last inequality uses the fact that $\| h(\cdot, \vx^{(Kt+l)}_t) \|_{\cH} \leq \xi$ almost surely as per Lemma \ref{lem:vpsvgd-itr-bounds}. 

To control the off-diagonal terms, let $i = Kt+l_1$ and $j = Kt+l_2$ for any arbitrary $l_1, l_2$ with $0 \leq l_1 \neq l_2 \leq K-1$. Conditioned on $\cF_t$, $\vx^{(i)}_t$ and $\vx^{(j)}_t$ are i.i.d samples from $\mu_t | \cF_t$. Thus, by Lemma \ref{lem:pettis-aux-lem} and Fubini's Theorem, 
\begin{align*}
    \bE_{\vx^{(i)}_t, \vx^{(j)}_t}\left[\dotp{h(\cdot, \vx^{(i)}_t)}{h(\cdot, \vx^{(j)}_t)}_\cH \bigr| \cF_t \right] &= \bE_{\vx^{(i)}_t}\left[\bE_{\vx^{(j)}_t}\left[\dotp{h(\cdot, \vx^{(i)}_t)}{h(\cdot, \vx^{(j)}_t)}_\cH \bigr|  \right]\right] \\
    &= \bE_{\vx^{(i)}_t}\left[\dotp{h_{\mu_t | \cF_t}}{h(\cdot, \vx^{(i)}_t)}_\cH \bigr|\cF_t\right] \\
    &= \| h_{\mu_t | \cF_t} \|^2_\cH
\end{align*}
Thus, we conclude that,
\begin{align*}
    \bE\left[ \| g_t \|^{2}_{\cH} | \cF_t\right] \leq \| h_{\mu_t | \cF_t } \|^{2}_{\cH} +  \frac{\xi^2}{K}
\end{align*}
\end{proof}
\subsection{Proof of Theorem \ref{thm:vpsvgd-bound}}
\label{prf:vpsvgd-bound-thm}
\begin{proof}
Let $\xi = \zeta_0 + \zeta_1 (\gamma T)^{\nicefrac{1}{\alpha}} + \zeta_2(\gamma^2 T)^{\nicefrac{1}{\alpha}} + \zeta_3 R^{\nicefrac{2}{\alpha}}$ where $\zeta_0, \dots, \zeta_3$ are as defined in Lemma \ref{lem:vpsvgd-itr-bounds}. Since $\gamma \leq \nicefrac{1}{2A_1 L}$, $\| g_t \|_{\cH} \leq \xi$ holds almost surely as per Lemma \ref{lem:vpsvgd-itr-bounds}. 

Since $\gamma \xi \leq \nicefrac{1}{2B}$, Lemma \ref{lem:evi-lemma} ensures that the following holds almost surely
\begin{align*}
    \KL{\mu_{t+1} | \cF_{t+1}}{\pistar} \leq \KL{\mu_{t} | \cF_{t}}{\pistar} - \gamma \dotp{h_{\mu_t | \cF_t}}{g_t}_{\cH} + \frac{\gamma^2 (4 + L)B}{2} \|g_t\|^{2}_{\cH}
\end{align*}
Taking conditional expectations w.r.t $\cF_t$ on both sides and applying Lemma \ref{lem:pettis-exp}, we obtain,
\begin{align*}
    \bE\left[ \KL{\mu_{t+1} | \cF_{t+1}}{\pistar} | \cF_t \right] &\leq \KL{\mu_{t} | \cF_{t}}{\pistar} - \gamma \left(1 - \frac{\gamma (4 + L) B}{2}\right) \| h_{\mu_t | \cF_t} \|^{2}_{\cH}  + \frac{\gamma^2 (4 + L)B \xi^2}{2 K} \\
    &\leq \KL{\mu_{t} | \cF_{t}}{\pistar} - \frac{\gamma}{2} \| h_{\mu_t | \cF_t} \|^2 + \frac{\gamma^2 (4+L)B \xi^2}{2K} \\
    &= \KL{\mu_{t} | \cF_{t}}{\pistar} - \frac{\gamma}{2} \KSD{\pistar}{\mu_t | \cF_t}{\pistar}^2 + \frac{\gamma^2 (4+L)B \xi^2}{2K}
\end{align*}
where the second inequality uses the fact that $\gamma \leq \nicefrac{1}{(4+L)B}$. Taking expectations on both sides and rearranging,
\begin{align*}
    \frac{\gamma}{2} \bE\left[\KSD{\pistar}{\mu_t | \cF_t}{\pistar}^2\right] \leq \bE\left[\KL{\mu_{t} | \cF_{t}}{\pistar} - \KL{\mu_{t+1} | \cF_{t+1}}{\pistar}\right] +  \frac{\gamma^2 (4+L)B \xi^2}{2K}
\end{align*}
Telescoping and averaging, we conclude,
\begin{equation}
\label{eq:vpsvgd-telescope-avg}
    \frac{1}{T} \sum_{t=0}^{T-1} \bE\left[\KSD{\pistar}{\mu_t | \cF_t}{\pistar}^2\right] \leq \frac{2\KL{\mu_{0} | \cF_{0}}{\pistar}}{\gamma T} +  \frac{\gamma (4+L)B \xi^2}{K}
\end{equation}
Now, recall from the proof of Lemma \ref{lem:evi-lemma} in Section \ref{prf:vpsvgd-evi-lemma-proof} that for any $t \in [T]$ and $l \geq Kt$, $\vx^{(l)}_t$ depends only on $\vx^{(0)}_0, \dots, \vx^{(Kt-1)}_0, \vx^{(l)}_0$, i.e., there exists a deterministic measurable function $H_t$ such that $\vx^{(l)}_t = H_t(\vx^{(0)}_0, \dots, \vx^{(Kt-1)}_0, \vx^{(l)}_0)$ holds almost surely. We note that the output $\vY = (\vy^{(0)}, \dots, \vy^{(n-1)})$ satisfies $\vy^{(l)} = \vx^{(KT+l)}_{S} \ \forall \ l \in (n)$, where $S \sim \mathsf{Uniform}((T))$ is sampled independently of everything else. 

Thus, we infer that $\vy^{(l)}$ depends only on $\vx^{(0)}_{0}, \dots, \vx^{(KT-1)}_0, S, \vx^{(KT+l)}_0$, i.e., there exists a deterministic measurable function $G$ such that $\vy^{(l)} = G(\vx^{(0)}_{0}, \dots, \vx^{(KT-1)}_0, S, \vx^{(KT+l)}_0)$ for every $l \in (n)$. Since $\vx^{(KT)}_0, \dots, \vx^{(KT+n-1)}_0 \ \iidsim \ \mu_0$, we infer that $\vy^{(0)}, \dots, \vy^{(n-1)}$ are i.i.d when conditioned on $\vx^{(0)}_{0}, \dots, \vx^{(KT-1)}_0, S$.

We now show that, when conditioned on  $\vx^{(0)}_{0}, \dots, \vx^{(KT-1)}_0, S$, $\vy^{(l)}$ is distributed as $\mubar$, where $\mubar$ is the probability kernel defined as $\mubar(\cdot; \vx^{(0)}_0 = x_0, \dots, \vx^{(KT-1)}_0 = \vx_{KT-1}, S=s) \coloneqq \mu_s(\cdot, \vx^{(0)}_0 = \vx_0, \dots, \vx^{(Ks-1)}_0 = \vx_{Ks-1})$. For any arbitrary fixed $s \in (T)$, note that, under the event $S=s$, $\vy^{(l)} = \vx^{(KT+l)}_s$ for every $l \in (n)$. Thus, for any Borel measurable set $A \subseteq \bR^d$, $\{ \vy^{(l)} \in A \} \cap \{S = s \} = \{ \vx^{(KT+l)}_s \in A \} \cap \{S = s \}$. For the sake of clarity, we denote the conditioning  $\vx^{(0)}_0 = \vx_0, \vx^{(KT-1)}_0 = \vx_{KT-1}$ as $\cC$, only in this proof. Since $S$ is independent of $\vx^{(l)}_t$ for every $t \in (T+1), l \in (KT+n)$, we infer the following:  
\begin{align*}
    \bP\left(\{\vy^{(l)} \in A\} \bigr| \cC, S=s\right) &= \frac{\bP\left(\{\vy^{(l)} \in A\} \cap \{ S=s\} \bigr| \cC\right)}{\bP\left(S=s\right)} \\
    &= T \bP\left(\{ \vx^{(KT+l)}_s \in A \} \cap \{ S = s \} \bigr| \cC \right) \\
    &= T \bP\left(\{ S = s \}\right)\bP\left(\{ \vx^{(KT+l)} \in A \} \bigr| \cC  \right) \\
    &= \bP\left(\{ \vx^{(KT+l)}_s \in A \} \bigr| \cC  \right) 
\end{align*}
As discussed above, $\vx^{(KT+l)}_s$ depends only on $\vx^{(0)}_0, \vx^{(Ks-1)}_0, \vx^{(KT+l)}_0$. It follows that $\bP\left(\{\vx^{(KT+l)}_s \in A \} | \cC \right) = \mu_s(A; \vx^{(0)}_0 = x_0, \dots, \vx^{(Ks-1)}_0 = x_{Ks-1})$ and,
\begin{align*}
    \bP\left(\{\vy^{(l)} \in A\} \bigr| \cC, S=s\right) &= \mu_s(A; \vx^{(0)}_0 = x_0, \dots, \vx^{(Ks-1)}_0 = x_{Ks-1}) \\
    &= \mubar(A; \vx^{(0)}_0 = x_0, \dots, \vx^{(KT-1)}_0 = x_{Kt-1}, S = s)
\end{align*}
Thus, $\vy^{(0)}, \dots, \vy^{(n-1)}$ are i.i.d samples from $\mubar$ when conditioned on $\vx^{(0)}_0, \dots \vx^{(KT-1)}_0, S$. 

We now obtain an upper bound on the expected squared KSD between $\mubar$ and $\pistar$. We recall from the proof of Lemma \ref{lem:evi-lemma} in Section \ref{prf:vpsvgd-evi-lemma-proof} that, for any $t \in (T+1)$, conditioned on $\vx^{(0)}_0, \dots, \vx^{(Kt-1)}_0$, $(\vx^{(l)}_t)_{l \geq t}$ are i.i.d samples from $\mu_t | \cF_t$ where $\mu_t | \cF_t \coloneqq \mu_t(\cdot; \vx^{(0)}_0, \vx^{(Kt-1)}_0)$. Hence, from \eqref{eq:vpsvgd-telescope-avg}, we conclude that,
\begin{align*}
    \bE[\KSD{\pistar}{\mubar(\cdot; (\vx^{(l)}_0)_{l \in (KT)}, S)}{\pistar}^2] &= \frac{1}{T}\sum_{t=0}^{T-1} \bE\left[\bE\left[ \KSD{\pistar}{\mubar(\cdot; (\vx^{(l)}_0)_{l \in (KT)}, S=t)}{\pistar}^2 \bigr|(\vx^{(l)}_0)_{l \in (KT)}\right]\right] \\
    &= \frac{1}{T}\sum_{t=0}^{T-1} \bE\left[\KSD{\pistar}{\mu_t(\cdot; \vx^{(0)}_0, \cdot, \vx^{(Kt-1)}_0)}{\pistar}^2 \right] \\
    &= \frac{1}{T} \sum_{t=0}^{T-1} \bE\left[\KSD{\pistar}{\mu_t | \cF_t}{\pistar}^2\right] \\
    &\leq \frac{2\KL{\mu_{0} | \cF_{0}}{\pistar}}{\gamma T} +  \frac{\gamma (4+L)B \xi^2}{K}
\end{align*}
where we use the fact that $S \sim \mathsf{Uniform}((T))$ is sampled independent of everything else.
\end{proof}
\subsection{Proof of Theorem \ref{thm:vpsvgd-hp-thm}}
\label{prf:vpsvgd-hp-bound-thm}
\begin{proof}
Following the same steps as Theorem \ref{thm:vpsvgd-bound}, we note that the following holds almost surely.   
\begin{align}
\label{eq:vpsvgd-hp-rate-descent}
    \KL{\mu_{t+1} | \cF_{t+1}}{\pistar} &\leq \KL{\mu_{t} | \cF_{t}}{\pistar} - \gamma \dotp{h_{\mu_t | \cF_t}}{g_t}_{\cH} + \frac{\gamma^2 (4 + L)B}{2} \|g_t\|^{2}_{\cH} \nonumber \\
    &\leq \KL{\mu_{t} | \cF_{t}}{\pistar} - \frac{\gamma}{2} \| h_{\mu_t | \cF_t} \|^{2}_{\cH} + \gamma \dotp{h_{\mu_t | \cF_t}}{h_{\mu_t | \cF_t} - g_t}_{\cH} \nonumber \\
    &+ \frac{\gamma^2 (4+L)B \xi^2}{2K} + \frac{\gamma^2 (4+L)B}{2} \left[ \|g_t\|^2_{\cH} - \|h_{\mu_t | \cF_t}\|^2_{\cH} - \frac{\xi^2}{K}\right]
\end{align}
where the last inequality uses the fact that $\gamma \leq \nicefrac{1}{(4+L)B}$. We now define $\Delta^{(l)}_{t}$, $\Delta_t$ and $r_t$ for $l \in (K), \ t \in (T)$ as follows:
\begin{align*}
    \Delta^{(l)}_{t} &= \dotp{h_{\mu_t | \cF_t}}{h_{\mu_t | \cF_t} - h(\cdot, \vx^{(Kt+l)}_t)}_{\cH} \\
    \Delta_t &= \frac{1}{K} \sum_{l=0}^{K-1} \Delta^{(l)}_t = \dotp{h_{\mu_t | \cF_t}}{h_{\mu_t | \cF_t} - g_t}_{\cH} \\
    r_t &= \| g_t \|^{2}_{\cH} - \| h_{\mu_t | \cF_t} \|^{2}_{\cH} - \frac{\xi^2}{K}
\end{align*}
Substituting the above into \eqref{eq:vpsvgd-hp-rate-descent}, we obtain the following:
\begin{align*}
    \KL{\mu_{t+1} | \cF_{t+1}}{\pistar} &\leq \KL{\mu_{t} | \cF_{t}}{\pistar} - \frac{\gamma}{2} \| h_{\mu_t | \cF_t} \|^{2}_{\cH} + \gamma \Delta_t + \frac{\gamma^2 (4+L)B\xi^2}{2K} + \frac{\gamma^2 (4+L)B r_t}{2} 
\end{align*}
Telescoping and averaging both sides, and using $\| h_{\mu_t | \cF_t} \|^{2}_{\cH} = \KSD{\pistar}{\mu_t | \cF_t}{\pistar}^2$, we obtain the following:
\begin{align}
\label{eq:vpsvgd-hp-telescope}
    \frac{1}{T} \sum_{t=0}^{T-1} \KSD{\pistar}{\mu_t | \cF_t}{\pistar}^2 &\leq  \frac{4 \KL{\mu_0 | \cF_0}{\pistar}}{\gamma T} + \frac{2 \gamma (4+L)B \xi^2}{K} \nonumber \\
    &+ \frac{4}{T} \sum_{t=0}^{T-1} \left( \Delta_t - \frac{\| h_{\mu_t | \cF_t} \|^{2}_{\cH}}{4}\right) + \frac{2 \gamma (4+L)B}{T} \sum_{t=0}^{T-1} r_t
\end{align}
We note that the first two terms are the same as that of the in-expectation guarantee for \ouralg~in Theorem \ref{thm:vpsvgd-bound}. The third and fourth term are random quantities that vanish in expectation. The remainder of our analysis upper bounds them with high probability.

We begin by deriving a high probability upper bound for the fourth term in \eqref{eq:vpsvgd-hp-telescope}. To this end, we note that, since $\gamma \leq \nicefrac{1}{2 A_1 L}$,  $\|h(\cdot, \vx^{(Kt+l)}_t)\|_{\cH} \leq \xi$ for any $t \in (T), \ l \in (K)$ as per Lemma \ref{lem:vpsvgd-itr-bounds}. Furthermore, since $\bE[h(\cdot, \vx^{(Kt+l)}_t | \cF_t] = h_{\mu_t | \cF_t}$ (both pointwise and in the sense of the Gelfand-Pettis integral, see proof of Lemma \ref{lem:pettis-exp} in Appendix \ref{prf:pettis-exp-lemma}), it follows by Jensen's inequality that $\| h_{\mu_t | \cF_t} \|_{\cH} \leq \xi$. This further implies that $|r_t| \leq 3 \xi^2$. Moreover, $r_t$ is $\cF_{t+1}$ measurable (as $g_t$ is an $\cF_{t+1}$ measurable random function) with $\bE[r_t | \cF_t] \leq 0$ (as per Lemma \ref{lem:pettis-exp})

Thus, $S_t = \sum_{s=0}^{t-1} r_t$ is an $\cF$-adapted supermartingale difference sequence with bounded increments. Thus, by the Hoeffding-Azuma inequality, we conclude that the following holds with probability at least $1 - \nicefrac{\delta}{2}$
\begin{align}
\label{eq:vpsvgd-hp-fourth-term-bound}
\frac{1}{T} \sum_{t=0}^{T-1} r_t \leq 6\xi^2 \sqrt{\frac{\log(\nicefrac{2}{\delta})}{T}}
\end{align}
We now proceed to control the third term in \eqref{eq:vpsvgd-hp-telescope}. Recall from the proof of Theorem \ref{thm:vpsvgd-bound} in Appendix \ref{prf:vpsvgd-bound-thm}, that, for any fixed $t \in (T)$, $(\vx^{(l)}_t)_{l \in (KT)}$ are i.i.d when conditioned on $\cF_t$. As discussed above, $\bE[h(\cdot, \vx^{(Kt+l)}_t)] = h_{\mu_t | \cF_t}$ in the sense of the Gelfand-Pettis integral, implying $\bE[\Delta^{(l)}_t] = 0$. Moreover, $|\Delta^{(l)}\|_t \leq 2 \xi \| h_{\mu_t | \cF_t}\|$. Thus, when conditioned on $\cF_t$, $\Delta^{(l)}_t$ are independent zero-mean bounded random variables. Hence, we conclude the following by Hoeffding's Lemma
\begin{align}
\label{eq:vpsvgd-hp-hoeffding-control}
\bE\left[e^{\theta \Delta_t} | \cF_t \right] &\leq \prod_{l=0}^{K-1} \bE[ e^{\tfrac{\theta \Delta^{(l)}_t}{K}}| \cF_t] \leq e^{\tfrac{2 \theta^2 \xi^2}{K} \| h_{\mu_t | \cF_t} \|^{2}_{\cH} }, \quad \forall \ \theta \in \bR 
\end{align}
We now define the sequence $M_t$ as follows, where $\lambda = \nicefrac{K}{8\xi^2}$ 
\begin{align*}
    M_t = \exp(\sum_{s=0}^{t-1} \lambda \Delta_s - \tfrac{\lambda}{4} \| h_{\mu_s | \cF_s} \|^{2}_{\cH})
\end{align*}
Since $g_t$ is $\cF_{t+1}$ measurable, so is $\Delta_t$, which implies $M_t$ is $\cF_{t+1}$ measurable. Furthermore,
\begin{align*}
    \bE[M_t | \cF_t] &= M_{t-1} e^{-\tfrac{\lambda}{4} \| h_{\mu_t | \cF_t} \|^{2}_{\cH}} \bE[ e^{\lambda \Delta_t} | \cF_t ] \\
    &\leq M_{t-1} e^{(-\tfrac{\lambda}{4} + \tfrac{2\lambda^2 \xi^2}{K}) \| h_{\mu_t | \cF_t} \|^{2}_{\cH}} \leq M_{t-1}
\end{align*}
Thus, $M_t$ is an $\cF$-adapted supermartingale sequence. Following the same steps, we conclude $E[M_1] \leq 1$, which implies $\bE[M_T] \leq \bE[M_1] \leq 1$. Thus, from Markov's Inequality
\begin{align*}
    \bP\left[ \sum_{t=0}^{T-1} \Delta_t - \tfrac{1}{4} \| h_{\mu_t | \cF_t}\|^{2}_{\cH} > x\right] \leq e^{-\lambda x}\bE[M_T] \leq e^{-\lambda x}
\end{align*}
Hence, the following holds with probability at least $1 - \nicefrac{\delta}{2}$. 
\begin{align}
\label{eq:vpsvgd-hp-supermartingale-control}
\sum_{t=0}^{T-1} \Delta_t - \tfrac{1}{4} \| h_{\mu_t | \cF_t}\|^{2}_{\cH} \leq \frac{8 \xi^2}{K} \log(\nicefrac{2}{\delta})
\end{align}
Substituting \eqref{eq:vpsvgd-hp-hoeffding-control} and \eqref{eq:vpsvgd-hp-supermartingale-control} into \eqref{eq:vpsvgd-telescope-avg} and taking a union bound, we conclude that the following holds with probability at least $1 - \delta$:
\begin{align}
\label{eq:vpsvgd-hp-avg-rate}
    \frac{1}{T} \sum_{t=0}^{T-1} \KSD{\pistar}{\mu_t | \cF_t}{\pistar}^2 &\leq \frac{4 \KL{\mu_0 | \cF_0}{\pistar}}{\gamma T} + \frac{2 \gamma (4+L)B \xi^2}{K} \nonumber \\
    &+ \frac{32 \xi^2 \log(\nicefrac{2}{\delta})}{KT} 
    +  12 \gamma (4+L) B \xi^2 \sqrt{\frac{\log(\nicefrac{2}{\delta})}{T}}
\end{align}
Recall from the proof of Theorem \ref{thm:vpsvgd-bound} in Appendix \ref{prf:vpsvgd-bound-thm} that the outputs $(\vy^{(l)})_{l \in (n)}$ of \ouralg~are i.i.d samples from the random measure $\mubar(\cdot; \vx^{(0)}_0, \dots, \vx^{(0)}_{KT-1}, S)$ when conditioned on $\vx^{(0)}_0, \dots, \vx^{(0)}_{KT-1}, S$. Furthermore, when conditioned on $S = t$, $\mubar(\cdot; \vx^{(0)}_0, \dots, \vx^{(0)}_{KT-1}, S= t) = \mu_t | \cF_t $. Thus, from \eqref{eq:vpsvgd-hp-avg-rate}, we conclude that, upon taking an expectation over $S \sim \mathsf{Uniform}((T))$ while conditioning on the virtual particles $\vx^{(0)}_0, \dots, \vx^{(KT-1)}_{0}$, the following holds with probability at least $1 - \delta$:
\begin{align*}
    \bE_S[\KSD{\pistar}{\mubar(\cdot; \vx^{(0)}_0, \dots, \vx^{(KT-1)}_0, S)}{\pistar}^2] &\leq \frac{1}{T} \sum_{t=0}^{T-1} \KSD{\pistar}{\mu_t | \cF_t}{\pistar}^2 \\
    &\leq \frac{4 \KL{\mu_0 | \cF_0}{\pistar}}{\gamma T} + \frac{2 \gamma (4+L)B \xi^2}{K}  \\
    &+ \frac{32 \xi^2 \log(\nicefrac{2}{\delta})}{KT} 
    +  12 \gamma (4+L) B \xi^2 \sqrt{\frac{\log(\nicefrac{2}{\delta})}{T}}
\end{align*}
\end{proof}
\section{Analysis of \ouralgnew}
\label{app-sec:gbsvgd-analysis}
In this section, we present our analysis of \ouralgnew. For any $t \in (T)$, we use $\tilg_t$ to denote the random function $\tilg_t(\vx) = \tfrac{1}{K} \sum_{r \in \cK_t} h(\vx, \vx^{(r)}_t)$ where $\cK_t$ is the random batch of size $K$ sampled at time-step $t$ of \ouralgnew. 

In order to prove Theorem \ref{thm:rb-svgd-bounds}, we first establish an almost-sure iterate bound for \ouralgnew~which is similar to that of Lemma \ref{lem:vpsvgd-itr-bounds} for \ouralg.
\begin{lemma}[Almost-Sure Iterate Bounds]
\label{lem:gbsvgd-itr-bounds}
Let Assumptions \ref{as:smooth}, \ref{as:growth}, \ref{as:norm}, \ref{as:kern_decay} and \ref{as:init} be satisfied. Then, the following holds almost surely for any $s \in \mathbb{N} \cup \{0\}$ and $t \in (T+1)$ whenever $\gamma \leq \nicefrac{1}{2 A_1 L}$
\begin{align*}
    \| \vx^{(s)}_t \| &\leq \zeta_0 + \zeta_1 (\gamma T)^{\nicefrac{1}{\alpha}} + \zeta_2 (\gamma^2  T)^{\nicefrac{1}{\alpha}} + \zeta_3 R^{\nicefrac{2}{\alpha}} \\
    \| h(\cdot, \vx^{(s)}_t) \|_{\cH} &\leq \zeta_0 + \zeta_1 (\gamma T)^{\nicefrac{1}{\alpha}} + \zeta_2 (\gamma^2  T)^{\nicefrac{1}{\alpha}} + \zeta_3 R^{\nicefrac{2}{\alpha}} \\
    \| \tilg_t \|_{\cH} &\leq \zeta_0 + \zeta_1 (\gamma T)^{\nicefrac{1}{\alpha}} + \zeta_2 (\gamma^2  T)^{\nicefrac{1}{\alpha}} + \zeta_3 R^{\nicefrac{2}{\alpha}}
\end{align*}
where $\zeta_0, \dots, \zeta_3$ are problem-dependent constants that depend polynomially on $A_1, A_2, A_3, B, d_1, d_2, L$ for any fixed $\alpha$.
\end{lemma}
\begin{proof}
The proof of this Lemma is identical to that of Lemma \ref{lem:vpsvgd-itr-bounds}. To this end, let $c^{(s)}_t = \tfrac{1}{K} \sum_{r \in \cK_t} k(\vx^{(s)}_t, \vx^{(r)}_t)$. Note that by Assumption \ref{as:kern_decay}, $c^{(s)}_t \geq 0$ Since $\vx^{(s)}_{t+1} = \vx^{(s)}_t - \gamma \tilg_t(\vx^{(s)}_t)$, it follows from the smoothness of $F$ that,
\begin{equation}
\label{eq:gbsvgd-path-descent}
    F(\vx^{(s)}_{t+1}) - F(\vx^{(s)}) \leq - \gamma \dotp{\nabla F(\vx^{(s)}_t)}{\tilg_t(\vx^{(s)}_t)} + \frac{\gamma^2 L}{2} \| \tilg_t(\vx^{(s)}_t) \|^2 
\end{equation}
By Lemma \ref{lem:h-bounds}, we note that,
\begin{align}
\label{eq:gbsvgd-descent-term1}
- \gamma \dotp{\nabla F(\vx^{(s)}_t)}{\tilg_t(\vx^{(s)}_t)} &= -\frac{\gamma}{K} 
\sum_{r \in \cK_t}\dotp{\nabla F(\vx^{(s)}_t)}{h(\vx^{(s)}_t, \vx^{(r)}_t)} \nonumber \\
&\leq \frac{\gamma}{K} \sum_{r \in \cK_t} \left[-\tfrac{1}{2}k(\vx^{(s)}_t, \vx^{(r)}_t) \| \nabla F(\vx^{(s)}_t)\|^2 + L^2 A_1 + A_3 \right] \nonumber \\
&\leq -\frac{\gamma c^{(s)}_t}{2} \| \nabla F(\vx^{(s)}_t)\|^2 + \gamma L^2 A_1 + \gamma A_3 
\end{align}
Moreover, by Jensen's Inequality and Lemma \ref{lem:h-bounds}
\begin{align}
\label{eq:gbsvgd-descent-term2}
\| \tilg_t(\vx^{(s)}_t) \|^2 &\leq \frac{1}{K} \sum_{r \in \cK_t} \| h(\vx^{(s)}_t, \vx^{(r)}_t) \|^2 \nonumber \\
&\leq \frac{1}{K} \sum_{r \in \cK_t} 2(\nicefrac{A_1 L}{2} + A_2)^2 + 2 k(\vx^{(s)}_t, \vx^{(r)}_t)^2 \| F(\vx^{(s)}_t) \|^2 \nonumber \\
&\leq \frac{1}{K} \sum_{r \in \cK_t} 2(\nicefrac{A_1 L}{2} + A_2)^2 + 2 A_1 k(\vx^{(s)}_t, \vx^{(r)}_t) \| F(\vx^{(s)}_t) \|^2 \nonumber \\
&\leq 2(\nicefrac{A_1 L}{2} + A_2)^2 + 2 A_1 c^{(s)}_t\| F(\vx^{(s)}_t) \|^2
\end{align}
Substituting \eqref{eq:gbsvgd-descent-term1} and \eqref{eq:gbsvgd-descent-term2} into \eqref{eq:gbsvgd-path-descent}, we obtain,
\begin{align*}
    F(\vx^{(s)}_{t+1}) - F(\vx^{(s)}_t) &\leq -\frac{\gamma c^{(s)}_t}{2} \| \nabla F(\vx^{(s)}_t)\|^2 + \gamma L^2 A_1 + \gamma A_3 \\
    &+ \gamma^2 L(\nicefrac{A_1 L}{2} + A_2)^2 + \gamma^2 L A_1 c^{(s)}_t\| F(\vx^{(s)}_t) \|^2 \\
    &\leq -\frac{\gamma c^{(s)}_t}{2} (1 - 2 A_1 L \gamma) \| \nabla F(\vx^{(s)}_t) \|^2 + \gamma A_3 + \gamma L^2 A_1 + \gamma^2 L (\nicefrac{A_1 L}{2} + A_2)^2 \\
    &\leq \gamma A_3 + \gamma L^2 A_1 + \gamma^2 L (\nicefrac{A_1 L}{2} + A_2)^2
\end{align*}
where the last inequality uses the fact that $c^{(s)}_t \geq 0$ and $\gamma \leq \nicefrac{1}{2 A_1 L}$. Now, iterating through the above inequality, we obtain the following for any $t \in [T], \ s \in \mathbb{N} \cup \{ 0 \}$
\begin{equation}
\label{eq:gbsvgd-descent-unroll}    
    F(\vx^{(s)}_t) \leq F(\vx^{(s)}_0) + \gamma T L^2 A_1 + \gamma T A_3 + \gamma^2 T L (\nicefrac{A_1 L}{2} + A_2)^2
\end{equation}
Furthermore, by Assumption \ref{as:smooth}
\begin{align*}
    F(\vx^{(s)}_0) &\leq F(0) + \| \nabla F(0) \| \| \vx^{(s)}_0 \| + \frac{L}{2} \| \vx^{(s)}_0 \|^2 \\
    &\leq F(0) + \nicefrac{1}{2} + L \| \vx^{(s)}_0 \|^2
\end{align*}
Substituting the above inequality into \eqref{eq:gbsvgd-descent-unroll}, and using Assumption \ref{as:growth}, we obtain the following for any $t \in [T], \ s \in \mathbb{N} \cup \{0\}$
\begin{align*}
    d_1 \| \vx^{(s)}_t \|^\alpha - d_2 \leq F(\vx^{s}_t) &\leq F(0) + \nicefrac{1}{2} + L \| \vx^{(s)}_0 \|^2 + \gamma T L^2 A_1 + \gamma T A_3 \\
    &+ \gamma^2 T L (\nicefrac{A_1 L}{2} + A_2)^2
\end{align*}
Rearranging and applying Assumption \ref{as:init}, we obtain
\begin{align*}
    \| \vx^{(s)}_t \| &\leq d^{-\nicefrac{1}{\alpha}}_1 \left[ F(0) + \nicefrac{1}{2} + L R^2 + \gamma T L^2 A_1 + \gamma T A_3 + \gamma^2 T L(A_1L+A_2)^2  \right]^{\nicefrac{1}{\alpha}} \\
    &\leq \tilde{\zeta}_0 + \tilde{\zeta}_1 (\gamma T)^{\nicefrac{1}{\alpha}} + \tilde{\zeta}_2 (\gamma^2 T)^{\nicefrac{1}{\alpha}} + \tilde{\zeta}_3 R^{\nicefrac{2}{\alpha}}
\end{align*}
where $\tilde{\zeta}_0, \dots, \tilde{\zeta}_3$ are constants that depend polynomially on $L, A_1, A_2, A_3, R$. We note that, since $0 < \alpha \leq 2$, the above inequality also holds for $t = 0$.

Using the above inequality Lemma \ref{lem:h-bounds} and Assumption \ref{as:smooth}, we conclude that the following holds almost surely for any $t \in (T+1), s \in \mathbb{N} \cup \{ 0 \}$
\begin{align*}
    \| h(\cdot , \vx^{(s)}_t) \|_{\cH} &\leq BL \| \vx^{(s)}_t \| + B\sqrt{L} + B \\
    &\leq \tilde{\eta_0} + \tilde{\eta}_1 (\gamma T)^{\nicefrac{1}{\alpha}} + \tilde{\eta}_2 (\gamma^2 T)^{\nicefrac{1}{\alpha}} + \tilde{\eta}_3 R^{\nicefrac{2}{\alpha}}
\end{align*}
where $\tilde{\eta}_0, \dots, \tilde{\eta}_3$ are constants that depend polynomially on $L, B, A_1, A_2, A_3, R$. Using the above inequality, we conclude that the following also holds for any $t \in (T+1)$.
\begin{align*}
    \| \tilg_t \|_{\cH} \leq \tilde{\eta_0} + \tilde{\eta}_1 (\gamma T)^{\nicefrac{1}{\alpha}} + \tilde{\eta}_2 (\gamma^2 T)^{\nicefrac{1}{\alpha}} + \tilde{\eta}_3 R^{\nicefrac{2}{\alpha}}
\end{align*}
Taking $\zeta_i = \max \{ \tilde{\zeta_i}, \tilde{\eta}_i \}$, the proof is complete. 
\end{proof}
\subsection{Proof of Theorem \ref{thm:rb-svgd-bounds}}
\label{prf:gbsvgd-bound-thm}
\begin{proof}
Let $\xi = \zeta_0 + \zeta_1 (\gamma T)^{\nicefrac{1}{\alpha}} + \zeta_2 (\gamma^2 T)^{\nicefrac{1}{\alpha}} + \zeta_3 R^{\nicefrac{2}{\alpha}}$ where $\zeta_0, \dots, \zeta_3$ are constants as described in Lemma \ref{lem:vpsvgd-itr-bounds} and Lemma \ref{lem:gbsvgd-itr-bounds}. Since the assumptions and parameter settings of Theorem \ref{thm:vpsvgd-bound} holds, $\gamma \leq \nicefrac{1}{2 A_1 L}$ and thus, by Lemma \ref{lem:vpsvgd-itr-bounds} and Lemma \ref{lem:gbsvgd-itr-bounds}, the particles output by \ouralg~and \ouralgnew~are bounded as $\| \vy^{(l)} \| \leq \xi$ and $\| \vybar^{(l)} \| \leq \xi$.

Let $\vY = (\vy^{(0)}, \dots, \vy^{(n-1)})$ and $\vYbar = (\vybar^{(0)}, \dots, \vybar^{(n-1)})$ denote the outputs of \ouralg~and \ouralgnew. Let $\muhat^{(n)} = \tfrac{1}{n} \sum_{i=0}^{n-1} \delta_{\vy^{(i)}}$ and $\nuhat^{(n)} = \tfrac{1}{n} \sum_{i=0}^{n-1} \delta_{\vybar^{(i)}}$ be their respective empirical distributions. We shall now explicitly construct a coupling between the inputs of \ouralg~and \ouralgnew~such that the first  $n-KT$ particles of their respective outputs are equal.
This in turn will allow us to control the expected squared KSD between $\muhat^{(n)}$ and $\nuhat^{(n)}$. 

To this end, let $\cE$ denote the event that each random batch $\cK_t$ of \ouralgnew~is disjoint and contains unique elements for every $t \in(T)$. Subsequently, let $\cK$ denote the set of all indices that were chosen to be part of some random batch $\cK_t$. Let $\Lambda$ be a uniformly random permutation over $\{0,\dots,n-1\}$. We note that, conditioned on $\cE$, the distribution of the random set $\mathcal{K}$ is the same as the distribution of $\{\Lambda(0),\dots,\Lambda(KT-1)\}$. We can couple a uniformly random permutation $\Lambda$ and $\mathcal{K}_t$ for $0\leq t\leq T$ such that under the event $\cE$, $\mathcal{K} =\{\Lambda(0),\dots,\Lambda(KT-1)\} $ and $\{\Lambda(tK),\dots,\Lambda((t+1)K-1)\}$ is the random batch $\mathcal{K}_t$. Thus, under the event $\cE$, one can couple a uniformly random permutation $\Lambda$ and $\cK_t$ for $t \in (T)$ such that $\cK = \{\Lambda(0),\dots,\Lambda(KT-1)\} $ and $\cK_t = \{\Lambda(tK),\dots,\Lambda((t+1)K-1)\}$ 

With this insight, we couple \ouralg~and \ouralgnew as follows. We note that, the random batch $\cK_t$ in \ouralgnew~is sampled independently of the initial particles. To this end, let $\vxbar_0^{(0)}, \dots, \vxbar_0^{(n-1)} \ \iidsim \ \mu_0$, and let the random batches $\cK_t$ and permutation $\Lambda$ be jointly distributed as described above, independently of $\vxbar_0^{(0)}, \dots, \vxbar_0^{(n-1)}$, i.e.
\begin{align*}
    \Lambda \sim \mathsf{Uniform}(\mathbb{S}_{(n)}), \quad \cK_t = \{\Lambda(tK),\dots,\Lambda((t+1)K-1)\}, \quad t \in (T)
\end{align*}

We now define $\vx_0^{(0)},\dots,\vx_0^{(KT+n-1)}$ as:
\begin{equation}
\label{eq:gbsvgd-couple-eq}
\vx_0^{(l)} \coloneqq 
\begin{cases}
  = \bar{\vx}_0^{(\Lambda(l))} & \textrm{ for }0 \leq l \leq n-1 \\
  \sim \mu_0 \text{ independent of everything else} & \textrm{ for } n\leq l \leq KT + n-1
\end{cases}
\end{equation}

Let $\bar{\vx}_0^{(0)},\dots,\bar{\vx}_0^{(n-1)}$ and $\mathcal{K}_t$ as the initialization and random batches for \ouralgnew, and let $\vx_0^{(0)},\dots,\vx_0^{(KT+n-1)}$ be the initialization for \ouralgnew. We first show that this construction is indeed a valid coupling between \ouralg~and \ouralgnew. 

\begin{claim}
\label{claim:couple-claim-1}
Conditioned on $\cE$, the inputs to \ouralg~and \ouralgnew, as constructed above is a valid coupling, i.e., the marginal distribution of $\vx_0^{(0)},\dots,\vx_0^{(KT+n-1)}$ is equal to the distribution of initial particles in \ouralg, and the marginal distribution of $\vxbar^{(0)}_0, \dots, \vxbar^{(n-1)}_0, (\cK_t)_{t \in (T)}$ is the same as the distribution of initial particles and random batches in $\cK_t$
\end{claim}
\begin{proof}
By construction $\vxbar^{(0)}_0, \dots, \vxbar^{(n-1)}_0 \ \iidsim \ \mu_0$. Moreover, conditioned on $\cE$, the distribution of $\cK_t = \{\Lambda(tK),\dots,\Lambda((t+1)K-1)\}$, has the distribution of a uniform random batch of size $K$ since $\Lambda \sim \mathsf{Uniform}(\mathbb{S}_n)$. Furthermore, since $\Lambda$ is sampled independently of $\vxbar^{(0)}_0, \dots, \vxbar^{(n-1)}_0$, $\cK_t$ is independent of $\vxbar^{(0)}_0, \dots, \vxbar^{(n-1)}_0$ for any $t \in (T)$. Thus, the coupling constructed above has the correct marginal with respect to \ouralgnew.

To establish the same for \ouralg, we note that by \eqref{eq:gbsvgd-couple-eq}, $\vx^{(n)}_0, \dots, \vx^{(KT+n-1)}_0 \ \iidsim \mu_0$, \emph{sampled independently of everything else}. Moreover, since $\vxbar^{(0)}_0, \dots, \vxbar^{(n-1)}_0 \ \iidsim \ \mu_0$, we infer that $\vxbar^{(\Lambda(0))}_0, \dots, \vxbar^{(\Lambda(n-1))}_0 \ \iidsim \ \mu_0$ for any arbitrary permutation $\Lambda \in \mathbb{S}_n$. From this, and \eqref{eq:gbsvgd-couple-eq}, we conclude that $\vx^{(0)}_0, \dots, \vx^{(KT+n-1)}_0 \ \iidsim \ \mu_0$. Hence, the coupling constructed above has the correct marginal with respect to \ouralg.
\end{proof}
We now show that, under the constructed coupling, the time-evolution of the particles of \ouralg~and \ouralgnew~satisfy $\bar{\vx}_t^{(\Lambda(l))} = \vx_t^{(l)}, \ KT \leq l \leq n-1, t \in (T+1)$, when conditioned on the event $\cE$. 
\begin{claim}
\label{claim:couple-claim-2}
Let the inputs to \ouralg~and \ouralgnew~be coupled as per the construction above. Then, conditioned on the event $\cE$, the particles $\vx^{(s)}_t$ and $\vxbar^{(s)}_t$ of \ouralg~and \ouralgnew~respectively, satisfy $\bar{\vx}_t^{(\Lambda(l))} = \vx_t^{(l)}$ for every $ KT \leq l \leq n-1$ and $0\leq  t \leq T$
\end{claim}
\begin{proof}
We prove this by an inductive argument. Clearly, the claim holds for $t = 0$ by the construction of our coupling. Assume it holds for some arbitrary $t \in (T)$. Now, writing the update equation for \ouralgnew~for $KT \leq l \leq n-1$,
\begin{align*}
    \vxbar^{(\Lambda(l))}_{t+1} &= \vxbar^{(\Lambda(l))}_{t} - \frac{\gamma}{K} \sum_{r \in \cK_t} h(\vxbar^{(\Lambda(l))}_{t}, \vxbar^{(r)}_{t}) \\
    &= \vxbar^{(\Lambda(l))}_{t} - \frac{\gamma}{K} \sum_{l=0}^{K-1} h(\vxbar^{(\Lambda(l))}_{t}, \vxbar^{(\Lambda(Kt+l))}_{t}) \\
    &= \vx^{(l)}_t - \frac{\gamma}{K} \sum_{l=0}^{K-1} h(\vx^{(l)}_t, \vx^{(Kt+l)}_t) = \vx^{(l+1)}_t
\end{align*}
where the second equality uses the fact that $\mathcal{K}_t = \{\Lambda(tK),\dots,\Lambda((t+1)K-1)\}$ when conditioned on $\cE$ and the third equality uses the induction hypothesis $\vxbar^{(\Lambda(l))}_t =  \vx^{(l)}_t$ for $KT \leq l \leq n-1$. Hence, the claim is proven true by induction. 
\end{proof}
Equipped with the above coupling between the inputs of \ouralg~and \ouralgnew, one can now couple their outputs by sampling an $S \sim \mathsf{Uniform}((n))$ and using this sampled $S$ as the random timestep chosen by both \ouralg~(Step 6 in Algorithm \ref{alg:vp_svgd_simple}) and \ouralgnew~(Step 7 in Algorithm \ref{alg:rb_svgd}) that are run with the coupled input constructed above. It is easy to see that this results in a coupling of the outputs $\vY$ and $\vYbar$ of \ouralg~and \ouralgnew~respectively. Furthermore, by Claim \ref{claim:couple-claim-2}, we note that, conditioned on the event $\cE$, $\vy^{(l-TK)} = \bar{\vy}^{(\Lambda(l))}$ for every $KT \leq l \leq n-1$. We now define the permutation $\tau \in \mathbb{S}_{(n)}$ as follows:

\begin{equation}
\tau(\Lambda(l)) = 
\begin{cases}
  l +n-KT & \text{ for } 0 \leq l \leq KT-1 \\
  l-KT  & \text{ for } KT \leq l \leq n-1
\end{cases}
\end{equation}
It follows that $\vybar^{\tau(l)} = \vy^{(l)}$ for $KT \leq l \leq n-1$. Thus, by definition of Kernel Stein Discrepancy (Definition \ref{def:ksd}), we can infer that the following holds when conditioned on the event $\cE$
\begin{align}
\label{eq:gbsvgd-coupling-event}
    \bE[\KSDsq{\pistar}{\nuhat^{(n)}}{\muhat^{(n)}}|\cE] &= \bE\left[\| h_{\nuhat^{(n)}} - h_{\muhat^{(n)}} \|^{2}_{\cH} \ | \ \cE\right] \nonumber \\
    &= \bE\left[ \| \frac{1}{n} \sum_{l=0}^{n-1} h(\cdot, \vybar^{(l)}) - \frac{1}{n} \sum_{l=0}^{n-1} h(\cdot, \vy^{(l)}) \|^{2}_{\cH} \ | \ \cE \right] \nonumber \\
    &= \frac{1}{n^2} \bE[\| \sum_{l=0}^{n-1} h(\cdot, \vybar^{(\tau(l))}) - h(\cdot, \vy^{(l)}) \|^{2}_{\cH} \ | \ \cE ] \nonumber \\
    &= \frac{1}{n^2} \bE[\| \sum_{l=0}^{KT-1} h(\cdot, \vybar^{(\tau(l))}) - h(\cdot, \vy^{(l)}) \|^{2}_{\cH} \ | \ \cE] \nonumber \\
    &\leq \frac{K T}{n^2} \sum_{l=0}^{KT-1} \bE[ \| h(\cdot, \vybar^{(\tau(l))}) - h(\cdot, \vy^{(l)}) \|_{\cH}^{2} \ | \ \cE] \nonumber \\
    &\leq \frac{2K^2 T^2 \xi^2}{n^2}
\end{align}
where the second step uses the permutation invariance of summation, the third step uses the fact that $\vybar^{\tau(l)} = \vybar^{(l)}$ for $KT \leq l \leq n-1$, the fourth step uses the convexity of $\| \cdot \|^{2}_{\cH}$ and the last step uses the almost-sure iterate bounds of Lemma \ref{lem:vpsvgd-itr-bounds} and \ref{lem:gbsvgd-itr-bounds}

Under the event $\cE^c$, we directly apply the almost-sure iterate bounds of Lemma \ref{lem:vpsvgd-itr-bounds} and \ref{lem:gbsvgd-itr-bounds} to obtain the following:

\begin{align}
\label{eq:gbsvgd-not-coupling-event}
    \bE[\KSDsq{\pistar}{\nuhat^{(n)}}{\muhat^{(n)}}|\cE^c] &=\bE\left[\| h_{\muhat^{(n)}} - h_{\nuhat^{(n)}} \|^{2}_{\cH} \ | \ \cE^c \right] \nonumber \\
    &= \frac{1}{n^2} \bE[\| \sum_{l=0}^{n-1} h(\cdot, \vybar^{(l)}) - h(\cdot, \vy^{(l)}) \|^{2}_{\cH} \ | \ \cE^c ] \nonumber \\
    &\leq 2 \xi^2
\end{align}
From Equations \eqref{eq:gbsvgd-coupling-event} and \eqref{eq:gbsvgd-not-coupling-event}, it follows that:
\begin{align*}
    \bE[\KSDsq{\pistar}{\nuhat^{(n)}}{\muhat^{(n)}}] &= \bE[\KSDsq{\pistar}{\nuhat^{(n)}}{\muhat^{(n)}}|\cE] \bP(\cE) + \bE[\KSDsq{\pistar}{\nuhat^{(n)}}{\muhat^{(n)}}|\cE^c] \bP(\cE^c) \\
    &\leq \frac{2K^2 T^2 \xi^2}{n^2} \bP(\cE) + 2 \xi^2 \bP(\cE^c)
\end{align*}
Recall that $P(\cE) = 1$ under sampling without replacement and $P(\cE) = 1 - \tfrac{K^2 T^2}{n}$ under sampling with replacement. Thus, we conclude that the following holds under the constructed coupling of $\vY$ and $\vYbar$ 
\begin{align*}
    \bE[\KSDsq{\pistar}{\nuhat^{(n)}}{\muhat^{(n)}}] \leq \begin{cases}\frac{2K^2T^2 \xi^2}{n^2} \quad &\text{(without replacement sampling)}\\
    \frac{2K^2T^2\xi^2}{n^2} \left(1 - \frac{K^2 T^2}{n}\right) + \frac{2K^2 T^2 \xi^2}{n}  \quad &\text{(with replacement sampling)}
    \end{cases}
\end{align*}
\end{proof}
\section{Finite-Particle Convergence Guarantees for \ouralg~and \ouralgnew}
\label{app-sec:finite-particle-analysis}
In this section, we proove the finite-particle convergence rates of VP-SVGD and GB-SVGD. To this end, present the proof of Corollary \ref{cor:vpsvgd-finite-particle-improved} in Appendix \ref{prf:cor-vpsvgd-finite-particle-improved} and the proof of Corollary \ref{cor:gbsvgd-finite-particle-improved} in Appendix \ref{prf:cor-gbsvgd-finite-particle-improved}. Finally, we compare the oracle complexity (i.e., the number of evaluations of $\nabla F$) of VP-SVGD and GB-SVGD to that of SVGD in Appendix \ref{app-sec:oracle-complexity-comparison}
\subsection{\ouralg~: Proof of Corollary \ref{cor:vpsvgd-finite-particle-improved}}
\label{prf:cor-vpsvgd-finite-particle-improved}
\begin{proof}
Recall from Algorithm \ref{alg:vp_svgd_simple} that the outputs of VP-SVGD are $\vx^{(KT)}_{S}, \dots, \vx^{(KT+n-1)}_S$ where $S \sim \mathsf{Uniform}(\{0, \dots, T-1\})$. Hence, their empirical measure $\muhat^{(n)}$ is given by $\muhat^{(n)} = \tfrac{1}{n} \sum_{l=0}^{n-1} \delta_{\vx^{(KT+l)}_S}$. From the definition of the Kernel Stein Discrepancy (Definition \ref{def:ksd}), it follows that,
\begin{equation}
\label{eq:vpsvgd-improved-finite-ksd}
    \KSDsq{\pistar}{\muhat^{(n)}}{\pistar} = \| h_{\muhat^{(n)}} \|^{2}_{\cH} = \| \tfrac{1}{n} \sum_{l=1}^{N} h(\cdot, \vx^{(KT+l)}_S) \|^{2}_{\cH}
\end{equation}
For the sake of clarity, only in this proof, we use $\cC$ to denote the conditioning on the virtual particles $\vx^{(0)}_0, \dots, \vx^{(KT-1)}_0$. Now, consider any arbitrary $t \in \{0, \dots, T-1\}$. Taking conditional expectations on both sides of Equation \eqref{eq:vpsvgd-improved-finite-ksd} by conditioning on $\cC$ and the event $\{S = t\}$, we obtain the following:
\begin{align}
    \bE\left[ \KSDsq{\pistar}{\muhat^{(n)}}{\pistar} \ | \ \cC, S=t\right ] &= \bE\left[  \| \tfrac{1}{n} \sum_{l=0}^{n-1} h(\cdot, \vx^{(KT+l)}_S) \|^{2}_{\cH} \ | \ \cC, S=t\right] \nonumber \\
    &= \bE\left[  \| \tfrac{1}{n} \sum_{l=0}^{n-1} h(\cdot, \vx^{(KT+l)}_t) \|^{2}_{\cH} \ | \ (\vx^{(s)}_0)_{0 \leq s \leq KT-1}\right] 
\label{eq:vpsvgd-improved-exp-step1}
\end{align}
Recall from Equation \eqref{eq:vpsvgd-rand-func-rep} in Appendix \ref{prf:vpsvgd-evi-lemma-proof} that for any $l \in \{0, \dots, n-1\}$ $\vx^{(KT+l)}_t$ depends only on $\vx^{(0)}_0, \dots, \vx^{(Kt-1)}_0$ and $\vx^{(KT+l)}_0$. Furthermore, from Appendix \ref{prf:vpsvgd-evi-lemma-proof}, we recall that the filtration $\cF_t$ is defined as $\cF_t = \sigma(\{ \vx^{(0)}_0, \dots, \vx^{(Kt-1)}_0\})$. It follows that, 
\begin{align}
    \bE\left[  \| \tfrac{1}{n} \sum_{l=0}^{n-1} h(\cdot, \vx^{(KT+l)}_t) \|^{2}_{\cH} \ | \ (\vx^{(s)}_0)_{0 \leq s \leq KT-1}\right] &= \bE\left[  \| \tfrac{1}{n} \sum_{l=1}^{N} h(\cdot, \vx^{(KT+l)}_t) \|^{2}_{\cH} \ | \ (\vx^{(s)}_0)_{0 \leq s \leq Kt-1}\right] \nonumber \\
    &= \bE\left[  \| \tfrac{1}{n} \sum_{l=0}^{n-1} h(\cdot, \vx^{(KT+l)}_t) \|^{2}_{\cH} \ | \ \cF_t\right]
\label{eq:vpsvgd-improved-exp-step2}
\end{align}
To control $\bE\left[  \| \tfrac{1}{n} \sum_{l=0}^{n-1} h(\cdot, \vx^{(KT+l)}_t) \|^{2}_{\cH} \ | \ \cF_t\right]$, we apply the arguments used in the proof of Lemma \ref{lem:pettis-exp}. To this end, note that when conditioned on the virtual particles $\vx^{(0)}_0, \dots, \vx^{(Kt-1)}_0$, the particles $\vx^{(KT)}_t, \dots, \vx^{(KT+n-1)}_t \ \iidsim \ \mu_t | \cF_t$. Furthermore, since $\gamma \leq \nicefrac{1}{2A_1 L}$ (as per the parameter settings of Theorem \ref{thm:vpsvgd-bound}), $\|h(\cdot, \vx^{(KT+l)}_{t})\|_{\cH} \leq \xi \ \forall \ l \in (n)$ by Lemma \ref{lem:h-bounds}. Finally, $\bE[  h(\vx, \vx^{(KT+l)}_t) | \cF_t] = h_{\mu_t | \cF_t}(\vx) \ \forall \ l \in (n), \ \vx \in \bR^d$. Hence, from Lemma \ref{lem:pettis-aux-lem}, we conclude that $h_{\mu_t | \cF_t}$ is the Gelfand-Pettis integral of the map $\vx \to h(\vx, \vx^{(KT+l)}_t)$ with respect to the measure $\mu_t | \cF_t$, i.e.,
\begin{equation}
\label{eq:vpsvgd-improved-finite-gelfand}
    \bE[\dotp{h(\cdot, \vx^{(KT+l)}_t)}{f}_{\cH} | \cF_t] = \dotp{h_{\mu_t | \cF_t}}{f} \ \forall \ f \in \cH
\end{equation}
To control $\bE\left[  \| \tfrac{1}{n} \sum_{l=0}^{n-1} h(\cdot, \vx^{(KT+l)}_t) \|^{2}_{\cH} \ | \ \cF_t\right]$, we proceed as follows:
\begin{align*}
    \| \tfrac{1}{n} \sum_{l=0}^{n-1} h(\cdot, \vx^{(KT+l)}_t) \|^{2}_{\cH} &= \frac{1}{n^2} \sum_{l_1, l_2 = 0}^{n-1} \dotp{h(\cdot, \vx^{(Kt + l_1)}_t)}{h(\cdot, \vx^{(Kt + l_2)}_t)}_{\cH} \\
    &= \frac{1}{n^2} \sum_{l=0}^{m-1} \| h(\cdot, \vx^{(Kt+l)}_t) \|^2_\cH + \frac{1}{n^2} \sum_{0 \leq l_1 \neq l_2 \leq n-1} \dotp{h(\cdot, \vx^{(Kt + l_1)}_t)}{h(\cdot, \vx^{(Kt + l_2)}_t)}_\cH  \\
    &\leq \frac{\xi^2}{n} + \frac{1}{n^2} \sum_{0 \leq l_1 \neq l_2 \leq n-1} \dotp{h(\cdot, \vx^{(Kt + l_1)}_t)}{h(\cdot, \vx^{(Kt + l_2)}_t)}_\cH  
\end{align*}
where the last inequality uses the fact that $\| h(\cdot, \vx^{(Kt+l)}_t) \|_{\cH} \leq \xi$ almost surely as per Lemma \ref{lem:vpsvgd-itr-bounds}. 

To control the conditional expectation of the off-diagonal terms, let $i = Kt+l_1$ and $j = Kt+l_2$ for any arbitrary $l_1, l_2$ with $0 \leq l_1 \neq l_2 \leq n-1$. Conditioned on $\cF_t$, $\vx^{(i)}_t$ and $\vx^{(j)}_t$ are i.i.d samples from $\mu_t | \cF_t$. Thus, by Equation \eqref{eq:vpsvgd-improved-finite-gelfand} and Fubini's Theorem, 
\begin{align*}
    \bE_{\vx^{(i)}_t, \vx^{(j)}_t}\left[\dotp{h(\cdot, \vx^{(i)}_t)}{h(\cdot, \vx^{(j)}_t)}_\cH \bigr| \cF_t \right] &= \bE_{\vx^{(i)}_t}\left[\bE_{\vx^{(j)}_t}\left[\dotp{h(\cdot, \vx^{(i)}_t)}{h(\cdot, \vx^{(j)}_t)}_\cH \bigr| \cF_t  \right]\right] \\
    &= \bE_{\vx^{(i)}_t}\left[\dotp{h_{\mu_t | \cF_t}}{h(\cdot, \vx^{(i)}_t)}_\cH \bigr|\cF_t\right] \\
    &= \| h_{\mu_t | \cF_t} \|^2_\cH
\end{align*}
It follows that,
\begin{align*}
    \bE\left[  \| \tfrac{1}{n} \sum_{l=0}^{n-1} h(\cdot, \vx^{(KT+l)}_t) \|^{2}_{\cH} \ | \ \cF_t\right] \leq \| h_{\mu_t | \cF_t } \|^{2}_{\cH} +  \frac{\xi^2}{n}
\end{align*}
Substituting the above into equation \ref{eq:vpsvgd-improved-exp-step1} and equation \ref{eq:vpsvgd-improved-exp-step2}, we obtain the following:
\begin{align*}
    \bE\left[ \KSDsq{\pistar}{\muhat^{(n)}}{\pistar} \ | \ \cC, S=t\right ] &\leq \frac{\xi^2}{n} + \| h_{\mu_t | \cF_t } \|^{2}_{\cH} = \frac{\xi^2}{n} + \KSDsq{\pistar}{\mu_t | \cF_t}{\pi^*}
\end{align*}
where the second step applies Definition \ref{def:ksd}. Finally, taking expectations with respect to $\cC$ and $S \sim \mathsf{Uniform}(\{0, \dots, T-1\})$ on both sides of the above inequality, we get:
\begin{align*}
    \bE\left[\KSDsq{\pistar}{\muhat^{(n)}}{\pistar}\right] \leq \frac{\xi^2}{n} + \frac{1}{T} \sum_{t=0}^{T-1} \bE[\KSDsq{\pistar}{\mu_t | \cF_t}{\pistar}]
\end{align*}
Substituting the bound from Theorem \ref{thm:vpsvgd-bound} into the above inequality, we conclude that:
\begin{align*}
    \bE[\KSDsq{\pistar}{\muhat^{(n)}}{\pistar}] \leq \frac{\xi^2}{n} + \frac{2 \KL{\mu_0 | \cF_0}{\pistar}}{\gamma T} + \frac{\gamma B (4+L)\xi^2}{K}
\end{align*}
We note that for $\gamma = O(\tfrac{(Kd)^\eta}{T^{1-\eta}})$ and $R = \sqrt{\nicefrac{d}{L}}$, $\KL{\mu_0 | \cF_0}{\pistar} = O(d)$ by Lemma \ref{lem:init_kl} and 
\begin{align*}
    \xi^2 \leq 4 \zeta_0 + 4 \zeta_1 (\gamma T)^{\nicefrac{2}{\alpha}} + 4 \zeta_2 (\gamma^2 T)^{\nicefrac{2}{\alpha}} + 4 \zeta_3 R^{\nicefrac{4}{\alpha}} \leq O\left(\left(KdT\right)^{\tfrac{1}{1+\alpha}} + d^{\nicefrac{2}{\alpha}}\right)
\end{align*}
Furthermore, 
\begin{align*}
    \frac{2 \KL{\mu_0 | \cF_0}{\pistar}}{\gamma T} + \frac{\gamma B (4+L)\xi^2}{K} &\leq O\left(\frac{d}{\gamma T} + \frac{\gamma B (4+L) \xi^2}{2K}\right) \leq O\left(\frac{d^{1-\eta}}{(KT)^{\eta}}\right) \\
    &\leq O\left(\frac{d^{\tfrac{2 + \alpha}{2(1+\alpha)}}}{(KT)^{\tfrac{\alpha}{2(1+\alpha)}}}\right)
\end{align*}
It follows that,
\begin{align*}
    \bE[\KSDsq{\pistar}{\muhat^{(n)}}{\pistar}] \leq O\left(\frac{d^{\nicefrac{2}{\alpha}}}{n} + \frac{(KTd)^{\tfrac{1}{1+\alpha}}}{n} + \frac{d^{\tfrac{2 + \alpha}{2(1+\alpha)}}}{(KT)^{\tfrac{\alpha}{2(1+\alpha)}}}\right)
\end{align*}
$KT = d^{\tfrac{\alpha}{2+\alpha}} n^{\tfrac{2(1+\alpha)}{2+\alpha}}$, we conclude:
\begin{align*}
    \bE[\KSDsq{\pistar}{\muhat^{(n)}}{\pistar}] \leq O\left(\frac{d^{\tfrac{2}{2+\alpha}}}{n^{\tfrac{\alpha}{2+\alpha}}} + \frac{d^{\nicefrac{2}{\alpha}}}{n}\right)
\end{align*}
\end{proof}
\subsection{\ouralgnew~: Proof of Corollary \ref{cor:gbsvgd-finite-particle-improved}}
\label{prf:cor-gbsvgd-finite-particle-improved}
\begin{proof}
Let $\vYbar = (\vybar^{(0)}, \dots, \vybar^{(n-1)})$ denote the $n$ particles output by GB-SVGD and let $\nuhat^{(n)} = \tfrac{1}{n} \sum_{l=0}^{n-1} \delta_{\vybar^{(l)}}$ denote their empirical measure. Let $\cE$ denote the event that each random batch $\cK_t$ of \ouralgnew~is disjoint and contains unique elements for every $t \in (T)$. Moreover, let $\vY = (\vy^{(0)}, \dots, \vy^{(n-1)})$ denote the $n$ particles output by VP-SVGD, run with the parameter settings stated above, and coupled with $\vYbar$ as per the coupling constructed in the proof of Theorem \ref{thm:rb-svgd-bounds} in Appendix \ref{prf:gbsvgd-bound-thm}. Let $\muhat^{(n)} = \tfrac{1}{n} \sum_{l=0}^{n-1} \delta_{\vy^{(l)}}$ denote their empirical measure. By definition of Kernel Stein Discrepancy (Definition \ref{def:ksd}) and the convexity $\|\cdot\|^{2}_{\cH}$, it follows that:
\begin{align*}
    \bE[\KSD{\pistar}{\nuhat^{(n)}}{\pistar}] &= \bE[\| h_{\nuhat^{(n)}} \|^{2}_{\cH}] \\
    &= \bE[\| h_{\nuhat^{(n)}} - h_{\muhat^{(n)}} + h_{\muhat^{(n)}} \|^{2}_{\cH}] \\
    &\leq 2 \bE[\| h_{\nuhat^{(n)}} - h_{\muhat^{(n)}} \|^{2}_{\cH}] + 2\bE[\| h_{\muhat^{(n)}} \|^{2}_{\cH}] \\
    &= 2 \bE[\KSDsq{\pistar}{\nuhat^{(n)}}{\muhat^{(n)}}] + 2 \bE[\KSDsq{\pistar}{\muhat^{(n)}}{\pistar}]
\end{align*}
Substituting the bounds of Theorem \ref{thm:rb-svgd-bounds} and Corollary \ref{cor:vpsvgd-finite-particle-improved} into the above inequality, we conclude the following:
\begin{align*}
    \bE[\KSDsq{\pistar}{\nuhat^{(n)}}{\pistar}] \leq \frac{4 K^2 T^2 \xi^2}{n^2} \bP(\cE) + 4 \xi^2 \bP(\cE^c) + \frac{2\xi^2}{n} + \frac{4 \KL{\mu_0 | \cF_0}{\pistar}}{\gamma T} + \frac{2 \gamma B (4+L)\xi^2}{K}
\end{align*}
We recall that, $\bP(\cE) = 1$ under without-replacement sampling of the random batches $\cK_t$ and $\bP(\cE) = 1 - \nicefrac{K^2 T^2}{n}$ under with-replacement sampling. Thus, under without-replacement sampling, the following holds:
\begin{align*}
    \bE[\KSDsq{\pistar}{\nuhat^{(n)}}{\pistar}] \leq \frac{4 K^2 T^2 \xi^2}{n^2} + \frac{2\xi^2}{n} + \frac{4 \KL{\mu_0 | \cF_0}{\pistar}}{\gamma T} + \frac{2 \gamma B (4+L)\xi^2}{K}
\end{align*}
Moreover, the following holds under with-replacement sampling
\small
\begin{align*}
    \bE[\KSDsq{\pistar}{\nuhat^{(n)}}{\pistar}] \leq \frac{4 K^2 T^2 \xi^2}{n^2}\left(1 - \frac{K^2 T^2}{n}\right) + \frac{4 K^2 T^2 \xi^2}{n} + \frac{2\xi^2}{n} + \frac{4 \KL{\mu_0 | \cF_0}{\pistar}}{\gamma T} + \frac{2 \gamma B (4+L)\xi^2}{K}
\end{align*}
\normalsize
Now, let us consider GB-SVGD without replacement with $R = \sqrt{\nicefrac{d}{L}}$, $\gamma = O(\tfrac{(Kd)^{\eta}}{T^{1-\eta}})$ and $KT = n^{\nicefrac{1}{2}}$ It follows that $\KL{\mu_0 | \cF_0}{\pistar} = O(d)$ by Lemma \ref{lem:init_kl} and 
\begin{align*}
    \xi^2 &\leq 4 \zeta_0 + 4 \zeta_1 (\gamma T)^{\nicefrac{2}{\alpha}} + 4 \zeta_2 (\gamma^2 T)^{\nicefrac{2}{\alpha}} + 4 \zeta_3 R^{\nicefrac{4}{\alpha}} \\
    &\leq O\left(\left(KdT\right)^{\tfrac{1}{1+\alpha}} + d^{\nicefrac{2}{\alpha}}\right) \\
    &\leq O\left(d^{\nicefrac{2}{\alpha}} + d^{\tfrac{1}{1+\alpha}} n^{\tfrac{1}{2(1 + \alpha)}}\right)
\end{align*}
Furthermore, 
\begin{align*}
    \frac{4 \KL{\mu_0 | \cF_0}{\pistar}}{\gamma T} + \frac{2 \gamma B (4+L)\xi^2}{K} &\leq O\left(\frac{d}{\gamma T} + \frac{\gamma B (4+L) \xi^2}{2K}\right) \leq O\left(\frac{d^{1-\eta}}{(KT)^{\eta}}\right) \\
    &\leq O\left(\frac{d^{\tfrac{2 + \alpha}{2(1+\alpha)}}}{n^{\tfrac{\alpha}{4(1+\alpha)}}}\right)
\end{align*}
Hence, we conclude that,
\begin{align*}
        \bE[\KSDsq{\pistar}{\nuhat^{(n)}}{\pistar}] &\leq \frac{4 K^2 T^2 \xi^2}{n^2} + \frac{2\xi^2}{n} + \frac{4 \KL{\mu_0 | \cF_0}{\pistar}}{\gamma T} + \frac{2 \gamma B (4+L)\xi^2}{K} \\
        &\leq \frac{6\xi^2}{n} + \frac{4 \KL{\mu_0 | \cF_0}{\pistar}}{\gamma T} + \frac{2 \gamma B (4+L)\xi^2}{K} \\
        &\leq O\left(\frac{d^{\nicefrac{2}{\alpha}}}{n} + \frac{d^{\tfrac{1}{1+\alpha}}}{n^{\tfrac{1+2\alpha}{2(1+\alpha)}}} + \frac{d^{\tfrac{2+\alpha}{2(1+\alpha)}}}{n^{\tfrac{\alpha}{4(1+\alpha)}}}\right)
\end{align*}
\end{proof}
\subsection{Oracle Complexity of SVGD, VP-SVGD and GB-SVGD}
\label{app-sec:oracle-complexity-comparison}
We now compare the gradient oracle complexity, (i.e., the number of evaluations of $\nabla F$) of VP-SVGD (as implied by Corollary \ref{cor:vpsvgd-finite-particle-improved}) and GB-SVGD (as implied by Corollary \ref{cor:gbsvgd-finite-particle-improved}) with that of SVGD as implied by the state-of-the-art finite particle guarantee of \citet{shi2022finite}.

\subsubsection{SVGD}
From Equation \eqref{eq:svgd-update}, We note that $T$ steps of SVGD run with $n$ particles requires $n^2 T$ evaluations of $\nabla F$.
\paragraph{Subgaussian $\pistar$} For subgaussian $\pistar$, the finite-particle convergence rate obtained by \citet{shi2022finite} is $\KSD{\pistar}{\muhat^{(n)}_{\mathsf{SVGD}}}{\pistar} = \Otilde(\tfrac{\mathsf{poly}(d)}{\sqrt{\log \log n^{\Theta(\nicefrac{1}{d})}}})$, where $\muhat^{(n)}_{\mathsf{SVGD}}$ denotes the empirical measure of the $n$ particles output by SVGD. By carefully following the analysis of \citet{shi2022finite}, we infer that, to achieve $\KSD{\pistar}{\muhat^{(n)}_{\mathsf{SVGD}}}{\pistar} \leq \epsilon$, SVGD requires $T = \Otilde(\tfrac{\poly(d)}{\epsilon^2})$ and $n = \Otilde(\exp(\Theta(d e^{\tfrac{\poly(d)}{\epsilon^2}})))$. Thus the oracle complexity of SVGD (as implied by \citet{shi2022finite}) for achieving $\KSD{\pistar}{\muhat^{(n)}_{\mathsf{SVGD}}}{\pistar}$ is $\Otilde(\tfrac{\poly(d)}{\epsilon^2} \cdot \exp(\Theta(d e^{\tfrac{\poly(d)}{\epsilon^2}})))$

\subsubsection{VP-SVGD}
From Algorithm \ref{alg:vp_svgd_simple}, we note that $T$ steps of VP-SVGD run with $n$ particles and a batch-size of $K$ requires $K^2 T^2 + KTn$ evaluations of $\nabla F$. 

\paragraph{Subgaussian $\pistar$} For subgaussian $\pistar$, Corollary \ref{cor:vpsvgd-finite-particle-improved} implies a finite-particle convergence rate of $\bE[\KSDsq{\pistar}{\muhat^{(n)}}{\pistar}] = O(\tfrac{d^{\nicefrac{1}{2}}}{n^{\nicefrac{1}{2}}} + \tfrac{d}{n})$ (where $\muhat^{(n)}$ denotes the empirical measure of the $n$ particles output by VP-SVGD) assuming $KT = d^{\nicefrac{1}{2}} n^{\nicefrac{3}{2}}$. Hence, to achieve $\bE[\KSD{\pistar}{\muhat^{(n)}}{\pistar}] \leq \epsilon$, VP-SVGD requires $n = O(\tfrac{d}{\epsilon^4})$ and $KT = d^{\nicefrac{1}{2}}n^{\nicefrac{3}{2}} = \tfrac{d^2}{\epsilon^6}$. The resulting oracle complexity for achieving $\bE[\KSD{\pistar}{\muhat^{(n)}}{\pistar}] \leq \epsilon$ is $O(\tfrac{d^4}{\epsilon^{12}})$. Compared to the oracle complexity of SVGD obtained above, this is a \emph{double exponential improvement in both $d$ and $\nicefrac{1}{\epsilon}$}. Notably, the obtained oracle complexity guarantee \emph{completely eliminates the curse of dimensionality}.

\paragraph{Subexponential $\pistar$} For subexponential $\pistar$, Corollary \ref{cor:vpsvgd-finite-particle-improved} implies a finite-particle convergence rate of $\bE[\KSDsq{\pistar}{\muhat^{(n)}}{\pistar}] = O(\tfrac{d^{\nicefrac{2}{3}}}{n^{\nicefrac{1}{3}}} + \tfrac{d^2}{n})$ (where $\muhat^{(n)}$ denotes the empirical measure of the $n$ particles output by VP-SVGD) assuming $KT = d^{\nicefrac{1}{3}} n^{\nicefrac{4}{3}}$. Hence, to achieve $\bE[\KSD{\pistar}{\muhat^{(n)}}{\pistar}] \leq \epsilon$, VP-SVGD requires $n = O(\tfrac{d^2}{\epsilon^6})$ and $KT = d^{\nicefrac{1}{3}}n^{\nicefrac{4}{3}} = \tfrac{d^3}{\epsilon^8}$. The resulting oracle complexity for achieving $\bE[\KSD{\pistar}{\muhat^{(n)}}{\pistar}] \leq \epsilon$ is $O(\tfrac{d^6}{\epsilon^{16}})$. 

\subsubsection{GB-SVGD}
From Algorithm \ref{alg:rb_svgd}, we note that $T$ steps of GB-SVGD run with $n$ particles and a batch-size of $K$ requires $KTn$ evaluations of $\nabla F$. 

\paragraph{Subgaussian $\pistar$} For subgaussian $\pistar$, Corollary \ref{cor:gbsvgd-finite-particle-improved} implies a finite-particle convergence rate of $\bE[\KSDsq{\pistar}{\nuhat^{(n)}}{\pistar}] = O(\tfrac{d^{\nicefrac{2}{3}}}{n^{\nicefrac{1}{6}}} + \tfrac{d}{n})$ (where $\nuhat^{(n)}$ denotes the empirical measure of the $n$ particles output by GB-SVGD) assuming $KT =  n^{\nicefrac{1}{2}}$. Hence, to achieve $\bE[\KSD{\pistar}{\nuhat^{(n)}}{\pistar}] \leq \epsilon$, GB-SVGD requires $n = \tfrac{d^4}{\epsilon^{12}}$ and $KT = \sqrt{n} = \tfrac{d^2}{\epsilon^6}$. Under this setting, the oracle complexity of GB-SVGD as implied by Corollary \ref{cor:gbsvgd-finite-particle-improved} is $O(\tfrac{d^6}{\epsilon^{18}})$. Compared to the oracle complexity of SVGD obtained above, this is \emph{a double exponential improvement in both $d$ and $\nicefrac{1}{\epsilon}$}. Notably, the obtained oracle complexity guarantee \emph{completely eliminates the curse of dimensionality} 

\paragraph{Subexponential $\pistar$} For subexponential $\pistar$, Corollary \ref{cor:gbsvgd-finite-particle-improved} implies a finite-particle convergence rate of $\bE[\KSDsq{\pistar}{\nuhat^{(n)}}{\pistar}] = O(\tfrac{d^{\nicefrac{3}{4}}}{n^{\nicefrac{1}{8}}} + \tfrac{d^2}{n})$ (where $\nuhat^{(n)}$ denotes the empirical measure of the $n$ particles output by GB-SVGD) assuming $KT =  n^{\nicefrac{1}{2}}$. Hence, to achieve $\bE[\KSD{\pistar}{\nuhat^{(n)}}{\pistar}] \leq \epsilon$, GB-SVGD requires $n = \tfrac{d^6}{\epsilon^{16}}$ and $KT = \sqrt{n} = \tfrac{d^3}{\epsilon^8}$. Under this setting, the oracle complexity of GB-SVGD as implied by Corollary \ref{cor:gbsvgd-finite-particle-improved} is $O(\tfrac{d^9}{\epsilon^{24}})$.
\section{Literature Review}
\label{app-sec:more-lit-review}
Initial works on the analysis of SVGD such as \citet{liu2017stein, lu2019scaling, duncan2019geometry, chewi2020svgd, nusken2021stein} consider the continuous-time population limit, i.e., the limit of infinite particles and vanishing step-sizes. In this regime, \citet{liu2017stein, lu2019scaling, nusken2021stein} show that the behavior of SVGD is characterized by a Partial Differential Equation (PDE), and established asymptotic convergence of this PDE to the target distribution. The work of \citet{duncan2019geometry} proposes the Stein Logarithmic Sobolev Inequality which ensures exponential convergence of this PDE to the target distribution. However, characterizing the conditions under which this inequality holds is an open problem. The work of \citet{chewi2020svgd} show that the PDE governing SVGD in the continuous-time population limit can be interpreted as an approximate Wasserstein gradient flow of the Chi-squared divergence. To this end, \citet{chewi2020svgd} shows that the (exact) Wasserstein gradient flow of the Chi-squared divergence exhibits exponential convergence to the target distribution when $\pistar$ satisfies a Poincare Inequality. To the best of our knowledge, the first discrete-time non-asymptotic convergence result for population-limit SVGD was established in \citet{korba2020non}, where the authors interpreted population-limit SVGD as projected Wasserstein gradient descent. Their result relied on the assumption that the Kernel Stein Discrepancy to the target is uniformly bounded along the trajectory of SVGD, a condition which is hard to verify apriori. This result was significantly improved in \citet{salim2022convergence}, which established convergence of population-limit SVGD assuming the potential $F$ is smooth the target $\pistar \propto e^{-F}$ satisfies Talagrand's inequality $\mathsf{T}_1$, an assumption which is equivalent to subgaussianity of $\pistar$. This result was extended in \citet{sun2023convergence} to accommodate for potentials $F$ that satisfy a more general smoothness condition. 

In comparison to prior works on population-limit SVGD, the literature on finite-particle SVGD is relatively sparse. The works of \citet{liu2017stein} and \citet{gorham2020stochastic} establish that the dynamics of finite-particle SVGD asymptotically converge to that of population-limit SVGD in bounded Lipschitz distance and Wasserstein-1 distance respectively, as the number of particles approaches infinity. Under the stringent condition of bounded $F$ (which is violated in various scenarios, e.g. log-strongly concave $\pistar$), \citet{korba2020non} derived a non-asymptotic bound between the expected Wasserstein-2 distance between finite-particle SVGD and population-limit SVGD. The work of \citet{liu2023understanding} considers the special case of Gaussian initialization and bilinear kernels, and provides a finite-particle convergence guarantee for the first and second moments whenever the target density is also a Gaussian. To the best of our knowledge, \citet{shi2022finite} is the only prior work that explicitly establishes a complete non-asymptotic convergence guarantee of finite-particle SVGD for subgaussian targets, and shows that the empirical measure of SVGD run with $n$ particles converges to the target density in KSD at a rate of $O\left(\sqrt{\tfrac{\mathsf{poly}(d)}{\log \log n^{\Theta(\nicefrac{1}{d})}}}\right)$.

\section{Additional Experimental Details}
\label{app-sec:exp-details}
As discussed in Section \ref{sec:exps}, we benchmark SVGD and \ouralgnew~on the task of Bayesian Logistic Regression using the Covertype dataset, obtained from the \href{http://archive.ics.uci.edu/ml/datasets/covertype}{UCI Machine Learning Repository}, whch contains around $580,000$ datapoints with dimension $d=54$. We use $n=100$ particles for both SVGD and \ouralgnew, and use batch-size $K=40$ for \ouralgnew. For both SVGD and \ouralgnew, we use AdaGrad with momentum to set step-sizes as per the procedure used by \citet{liu2016stein}. Our experiments were performed using Python 3 on a 2.20 GHz Intel Xeon CPU with 13 GB of memory. 
\end{document}